\documentclass[11pt,oneside,reqno]{amsart}
\usepackage{mathptmx}
\usepackage{graphicx}
\usepackage{times}
\usepackage{hyperref}

\usepackage{amsmath,amsthm,amssymb,latexsym,amscd}

\usepackage{subfigure} %% if you want to group several picures together.
\usepackage{pstricks}  %% if you draw pictures in pstricks.
\usepackage{pst-node}  %% if you draw pictures in pstricks.
\usepackage{scalefnt}  %% if you want to scale the fonts (make it bigger or smaller).

\usepackage{etoolbox}% http://ctan.org/pkg/etoolbox
\makeatletter
\patchcmd{\l@chapter}{\bfseries}{}{}{}% \patchcmd{<cmd>}{<search>}{<replace>}{<success>}{<failure>}
\makeatother

\usepackage{algpseudocode}
\usepackage{algorithm}
\usepackage{enumitem}

\usepackage{subfiles}
\usepackage{graphicx}
\usepackage{subfigure}
\usepackage{epstopdf}
\usepackage{comment}
\usepackage{float}
\usepackage[mathscr]{euscript}
\usepackage{etoolbox}

\newcommand {\mm}[1] {\ifmmode{#1}\else{\mbox{\(#1\)}}\fi}

\def\mc{\mathcal}
\def\mb{\mathbb}
\def\rar{\rightarrow}

\usepackage{graphicx}
\usepackage{pict2e}
\usepackage{amssymb}
\usepackage{amsthm}                % better theorem environments
\usepackage[margin=1in]{geometry}  % set the margins to 1in on all sides
\usepackage{graphics}
\usepackage{epstopdf}
\usepackage{amsmath}
\usepackage{subfigure}
\usepackage{latexsym}
\usepackage{amsmath}
\usepackage{amsfonts}
\usepackage{amssymb}
\usepackage{tikz}
\usepackage{mathdots}
\usepackage{comment}
\usepackage{float}
\usepackage[mathscr]{euscript}
\usepackage{enumerate}
\usepackage{epsfig}
\usepackage { graphicx }   %redundant
\usepackage { hyperref }
\usepackage { graphics }   %redundant
\usepackage { graphicx }   %redundant

\usepackage[utf8]{inputenc}
\usepackage{longtable}
\usepackage[lite]{amsrefs}
\renewcommand{\PrintDOI}[1]{\href{http://dx.doi.org/\detokenize{#1}}{doi: \detokenize{#1}}%
	\IfEmptyBibField{pages}{, (to appear in print)}{}}

\theoremstyle{definition}
\newtheorem{theorem}{Theorem}[section]
\newtheorem{lemma}[theorem]{Lemma}

\theoremstyle{definition}
\newtheorem{definition}[theorem]{Definition}

\theoremstyle{remark}

\numberwithin{equation}{section}

\numberwithin{equation}{section}
\newcommand{\R}{\mathbb{R}}  % The real numbers.

\title{Graph Based Analysis for Gene Segment Organization In a Scrambled Genome}

\author[1]{Mustafa Hajij}

\address[1]{Department of Computer Science and Engineering, University of South Florida, Tampa, FL 33612, USA}

\author[2]{Nata\v{s}a Jonoska}

\author[2]{ 
Denys Kukushkin}

\author[2]{Masahico Saito}

\address[2]{Department of Mathematics and Statistics, University of South Florida, Tampa, FL 33612, USA}
\date{}

\thanks{NJ was partially supported by National Science Foundation  CCF-1526485 and National Institutes of Health  R01GM109459. MS was partially supported by  National Institutes of Health  R01GM109459.
}

\begin{document}

%	\maketitle 
	
\begin{abstract}
 DNA rearrangement processes recombine gene segments that are organized on the chromosome in a variety of ways. The segments can overlap, interleave or one may be a subsegment of another. We use directed graphs to represent segment organizations on a given locus where contigs containing rearranged segments represent vertices and the edges correspond to the segment relationships. Using graph properties we associate a point in a higher dimensional Euclidean space to each graph such that 
cluster formations and analysis can be performed with methods from topological data analysis.
The method is applied to a recently sequenced model organism \textit{Oxytricha trifallax}, a species of ciliate 
with highly scrambled genome that undergoes  massive rearrangement process after conjugation. The analysis shows some emerging star-like graph structures indicating that segments of 
a single gene can interleave, %with,
or even contain  all of the segments from fifteen or more other genes in between its segments. We also observe that as many as six genes can have their segments mutually interleaving or  overlapping.  
\end{abstract}

\maketitle

%\tableofcontents

%\maketitle

%\bigskip

\section{Introduction}
\label{sec:intro}

It has long been observed that genome rearrangement processes on an evolutionary scale can lead to speciation~\cite{dobzhansky1933sterility}, while on developmental scale they often involve DNA 
deletions~\cite{beermann1977diminution,shibata2012extrachromosomal} as well as wholescale programmed rearrangements~\cite{smith2012genetic,prescott1994dna}. 
For example, the highly diverse collection of antibodies often is attributed to somatic DNA 
recombination~\cite{tonegawa1983somatic}, and  rearrangements on a chromosomal levels can be observed during homologous recombination~\cite{rieseberg2001chromosomal}.
In recent years there are numerous observations of alternative splicing where rearranging patterns of exons and introns of a single gene can produce 
different protein variants from a single mRNA (e.g.~\cite{haussmann2016m6a}).  Rearranging segments of nucleotide sequences can be organized 
in a variety of arrangements on the locus, for example, they can be  overlapping or interleaving~\cite{g3}.
{\it Oxytricha trifallax} is a single cell organism that is often taken as a model organism to study DNA rearrangement processes. 
This, 
and similar species of ciliates undergo massive restructuring of a germline micronuclear DNA  during development of a somatic macronucleous. Recent sequencing and annotation of the whole {\it O. trifallax} 
genome allow genome level studies~\cite{chen2014architecture,burns2015database}. Scrambling patterns within  thousands of
 genes were observed revealing hidden structures among the 
scrambled gene/nanochromosome segments that explain over 95\% of the scrambled 
genome~\cite{burns2016recurring}. While those studies were focused on scrambled recurrent patterns  within 
a single gene, 
in this paper we study inter-gene segment arrangements. Through clustering techniques we identify patterns in which segments of different genes interleave or overlap throughout the genome.
We  represent a micronuclear locus with the interleaving and overlapping gene segments by a directed graph. 
Such obtained graph data is then converted to a set of points (point cloud) in a Euclidean space and
%by listing graph features (properties) to 
 we apply topological data analysis techniques to obtain clusters of similar graphs. This 
method was  applied to the whole genome data of \textit{Oxytricha trifallax}~\cite{chen2014architecture,burns2015database}.

In the last decade, Topological Data Analysis (TDA) has shown to be another tool for data analysis and data mining  %techniques 
 that can be used to extract topological information from various types of data~\cite{carlsson2009topology}. TDA originated from computational topology with ideas inspired from statistics and computer science.  One of the notable tools in TDA is \textit{persistence homology}. Persistence homology, or for brevity PH, was introduced by Edelsbrunner et al.~\cite{edelsbrunner2000topological} and later studied further by Zomorodian and Carlsson \cite{zomorodian2005computing}. The theory has been studied extensively since then and many theoretical advances have been made 
 %since it was introduced
\cite{carlsson2010zigzag,carlsson2009zigzag,chazal2013persistence}. Persistence Homology is defined in all dimensions $\ge 0$, but clustering analysis, as used in this paper, uses only dimension $0$.  More details on TDA and persistence homology can be found in \cite{carlsson2009topology,edelsbrunner2008persistent,ghrist2008barcodes}. 

 Due to advances in bioscience and biotechnology, the growth of biomolecular data has exploded and many data analysis algorithms have been developed aiming to better understand the generated 
 data~\cite{cheng2006dompro,fernandez2014improving,meinicke2015uproc}. Data analysis using topological methods has proven to be useful in showing 
general patterns that were difficult to  observe with other techniques. TDA methods have been used recently in many  applications including protein structure identification \cite{gameiro2015topological,xia2014persistent}, aggregation models for animal behavior \cite{ballerini2008interaction}, and fullerene stability \cite{xia2015persistent}.

%\subsection{Point Cloud Data Construction}

%The TDA analysis 
%starts by converting a set of points in an Eucledian space, called 
%{\it point cloud data}  into a nested sequence of simplicial complexes called a \textit{filtration} and then performing persistent homology computation on this filtration. 

Our data set consists of directed graphs $\mathcal{G}=\{G_1,\ldots,G_n\}$ which we convert 
to a set of points in a Euclidean space, point cloud. 
$S=\{P(G_1),\ldots,P(G_n)\}\subset \mathbb{R}^n$. For the clustering analysis we applied TDA with PH at dimension zero to the obtained point cloud in $\mb R^n$.
 This process is described 
 in Section~\ref{methods}.

%\begin{figure}[H]
%  \centering
%   {\includegraphics[scale=0.04]{pipline}
%   \caption{Three steps of abstracting the dataset. Left : a set of graphs obtained from the data.
   %that represents interaction types of  contigs. 
%   Middle : the graphs are associated with  points in a Euclidean space (point cloud). Right : we construct a complex filtration on the point cloud and extract  clusters at every level.}
%   \label{pipline}
%  }
%\end{figure}

Although this paper is focused on analysis of a 
%our main aim is to study 
specific data of interleaving/overlapping gene segments, the method that we propose for converting a graph data to a point cloud data % to apply TDA,  
is novel and general,  and can be applied to analyze similarities in 
 various graph data. 

%In our paper we utilize the zero-dimensional persistence homology. In this case PH can be considered a hierarchical clustering procedure on the underlying space. 

\section{Method and Data Construction}
\label{methods}

\subsection{Gene Segment Maps in \textit{Oxytricha trifallax}}

\textit{Oxytricha trifallax} is a species of ciliate used as a model organism to study genome rearrangements. It undergoes massive genome rearrangements during the development of a somatic macronucleus (MAC) specializing in gene expression from an archival germline micronucleus (MIC)~\cite{chen2014architecture}. 
Within this process, over 16,000 macronuclear 
nanochromosomes assemble through DNA processing events involving global deletion 
of 90-95\% of the germline DNA, effectively eliminating nearly all so-called 
``junk'' DNA, including intervening DNA segments (internally eliminated 
sequences, IESs).  Because these IES segments interrupt the coding regions of the precursor macronuclear gene loci in the 
micronucleus, 
%Because IESs interrupt coding regions, 
%Therefore, 
each macronuclear gene 
may appear as several nonconsecutive segments (macronuclear destined 
sequences, MDSs) in the micronucleus.  Moreover, the precursor order of 
these MDS segments for thousands of genes can be permuted or inverted in the micronucleus such that
during the macronuclear development, all IESs  are  
deleted and the MDSs are rearranged to form thousands of gene-sized chromosomes. 

In \cite{chen2014architecture} and later in \cite{g3}, it was 
observed that an IES between consecutive MDSs of one gene can contain
MDS segments from other genes, and that this process can be nested. Furthermore MDSs from different MAC genes can overlap or one MDS can be a subsegment of an MDS of another gene. % Fig.~\ref{precursorProduct}). 
\begin{figure}[h!]
\begin{center}
\includegraphics[scale=.25]{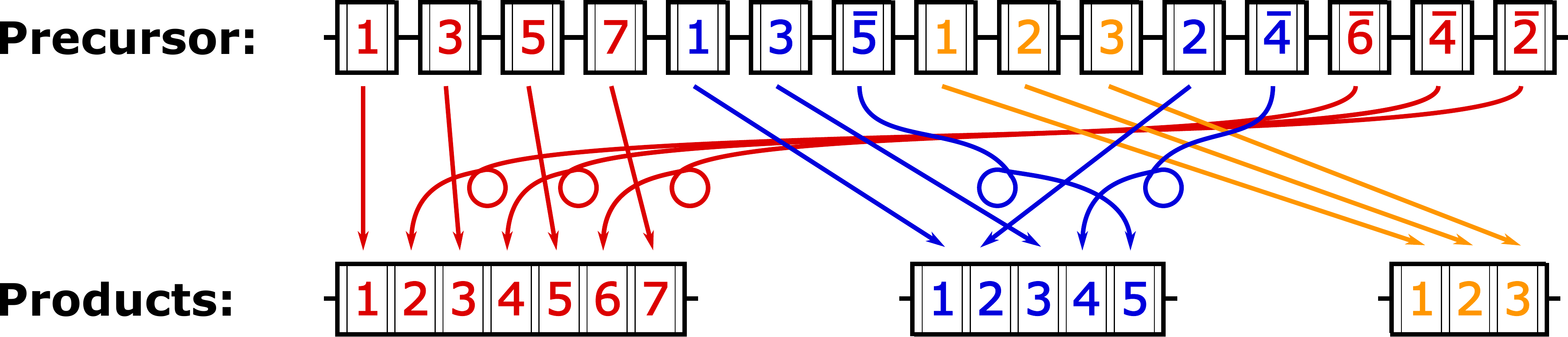}
\caption{Rearrangement of gene segments in {\it Oxytricha trifallax}}
\label{precursorProduct}
\end{center}
\end{figure}

The interleaving situation is schematically depicted in Figure~\ref{precursorProduct}.
%In the figure,  a MIC contig containing MDSs of three MAC contigs, each indicated with a different color, is depicted.
MDSs are
represented by colored boxes with numbers. 
 This example illustrates a MIC contig that 
has MDSs of three MAC contigs, each indicated with a different color. The numbers within the boxes indicate the order of the MDSs in the corresponding MAC contig. The barred numbers 
indicate MDSs in a reverse orientation 
(inverted) in the MIC contig relative to the ordering of the MDSs in the corresponding MAC contig. 
%This 
The thin
black lines between the squares %representing MDSs 
indicate IESs.

%http://trifallax.princeton.edu/cms/raw-data/genome/mic/Oxytricha_trifallax_micronuclear_genome_MDS_IES_maps.gff

%\url{http://trifallax.princeton.edu/cms/raw-data/genome/mic/Oxytricha_trifallax_micronuclear_genome_MDS_IES_maps.gff}

The  MDS interrelationship  analyzed in this paper uses  the genome sequencing data in \cite{chen2014architecture}, and can be downloaded  
from the Supplemental Information %(S1) 
in ~\cite{chen2014architecture} and also in \cite{burns2016recurring}.
The
%se situations are schematically depicted in 
 data used for analysis in this paper is the processed data reported in \cite{burns2016recurring}. This data  was
filtered so that consecutive MDSs of a single MAC
contig that overlap or have no nucleotide gap (are adjacent) in the MIC
are merged into a singe MDS (correcting  to possible previous annotation artifacts). In addition, we excluded MAC contigs that contain  
 segments that are distant in the MAC contig but overlap in a
MIC contig, as well as  MAC contigs that are  alternative fragmentations
of longer MAC contigs. Both of these cases could be considered artifacts of the MIC-MAC maps although we cannot exclude the possibility that these cases are genuine to the data. We refer to such processed data as data $\mc D$ available at \url{http://knot.math.usf.edu/data/scrambled_patterns/processed_annotation_of_oxy_tri.gff}.

%In summary, the following situations in nuclei of {\it Oxytricha Trifallax}.
%\begin{itemize}
%\item
%In the conjugation process, micronucleus (MIC) contig  is rearranged to form a macronucleus (MAC).

%\item
%Divided segments of MAC genes are located in MIC and  called MAC destined sequences (MDSs). They are separated by ``junk'' DNA called internally eliminated sequences
%(IESs). 

%\item 
%An MDS from a  MAC genes may appear in MIC  in between two MDSs of a different MAC gene ({\it interleaving} MDSs), and it can overlap with another MDS of a different gene
%({\it overlapping} MDSs).
%\end{itemize}

\subsection{Graphs Corresponding to MIC Contings}
As MDS from different MAC contigs can overlap or interleave in a MIC contig 
we define the following types of relationships between MAC contigs 
%The following types of  MDS relationships among
%MAC contigs  
located on a single MIC contig.
%have been observed.
%
\begin{figure}[h!]
\begin{center}
\includegraphics[scale=.30]{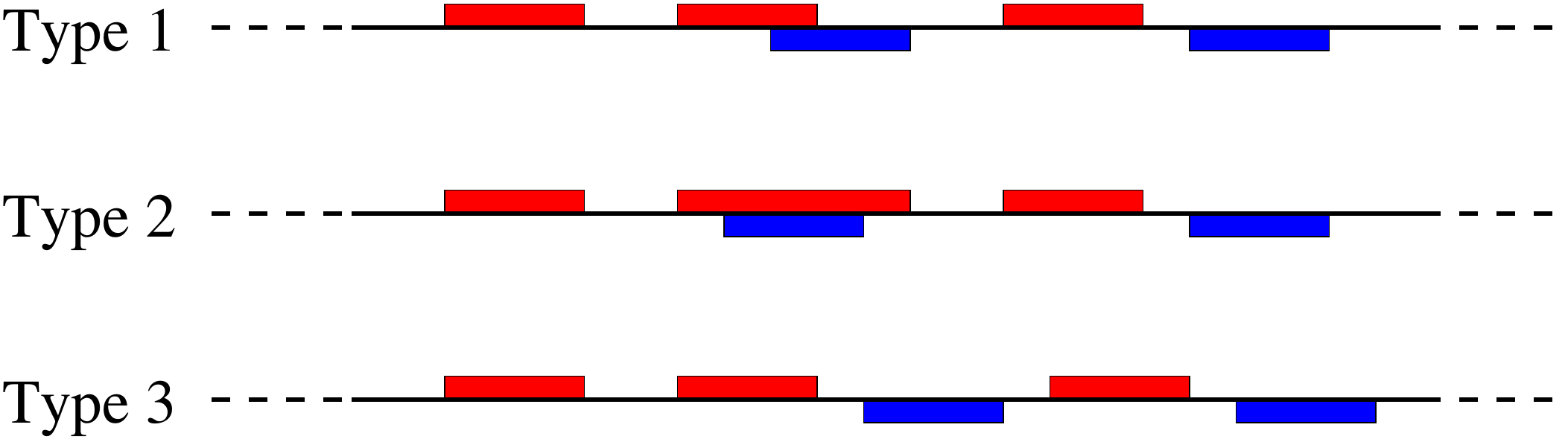}
\caption{Three types of MDS segment organization of different MAC contigs. MDSs of the same MAC contig are colored the same.}
\label{types}
\end{center}
\end{figure}

\begin{itemize}
\item (Type 1 : Overlapping)
If an MDS of a MAC contig $g_1$ overlaps with an MDS of
another, distinct MAC contig $g_2$, then it is said that $g_1$ and $g_2$ {\it overlap}, or 
they are {\it overlapping}. We also say that $g_1$ has type 1 interaction  with $g_2$,
or $g_1$ has interaction of type 1 with $g_2$. 

Two MAC contigs are considered to be overlapping if they have at least one pair of MDSs that overlap with at least 20bp in common. This is because two consecutive MDSs of the same MAC contig usually share sequences 
at their ends (pointers) that guide the rearrangement process~\cite{prescott1994dna}, and two MDSs from distinct MAC contigs can share the same pointer sequence. The length of these pointer sequences usually ranges between 2 to 20 nucleotides. 

%This number 20 bps was chosen in consultation with \cite{Laura}.

The overlapping relation is symmetric,  if $g_1$ overlaps with $g_2$, then  $g_2$ overlaps with $g_1$. 
The situation is depicted in  Figure~\ref{types} (Type 1).
In the figure, MDSs of $g_1$ and $g_2$ are represented by blue and red rectangles, respectively.
This case excludes the  case when one MDS is completely included in another, even though being a subsequence is a particular type of ``overlapping''. Such situation is included in Type 2 case (below).

\item (Type 2 : Containment) 
If an MDS of a MAC contig $g_1$ is contained  in (is a subsegment of) an MDS of
another distinct MAC contig $g_2$, then it is said that an MDS of $g_1$  
is contained in an MDS of $g_2$, and we say  $g_1$ has type 2 interaction with $g_2$. 

For this interaction when an MDS $M$ of $g_2$ contains an MDS $M^{\prime}$ of $g_1$, we require that both ends of $M$ have at least 5 bps that are not in common with $M^{\prime}$. That is, we require a complete inclusion such that there are no pointer sequences in common. In Figure~\ref{types} (Type 2), MDSs of $g_1$ are depicted in blue, and those for $g_2$ 
in red. This relation is not symmetric.
We distinguish this situation from the one in Type 1 because the unscrambling  of an MDS that is next to at least one IES (Type 1) may use a different  biological process involving Piwi-interacting RNA~\cite{Fang} rather than one that does not neighbor an IES ($g_1$ in Type 2).

\item (Type 3 : Interleaving)
If an IES of a MAC contig $g_1$ 
contains % is contained in %%%% Change definition according to Fig. 4.
an MDS of
another, distinct MAC contig $g_2$, then it is said that an MDS of $g_1$  
interleaves (or is interleaving) with %in 
$g_2$, or $g_1$ has Type 3 interaction with $g_2$. 
We allow that the ends of an interleaving MDS of $g_1$ and the MDSs of $g_2$ to intersect (overlap) up to (including) 5 bases. This requirement distinguishes type 3 case from the `overlapping' case where the requirement is at least 20 bases. 

\end{itemize}

We consider pairs $(g_1,g_2)$ of MAC contigs $g_1$ and $g_2$ that belong to the same MIC contig. To each pair of MAC contigs $(g_1,g_2)$ we associate a triple $c(g_1,g_2)=(b_1,b_2,b_3)$ where each entry $b_i$ ($i=1,2,3$) indicates whether $g_1$ is in relationship of Type $i$ with $g_2$. The value of $b_i$ is either $0$ (there is no relationship of Type $i$) or $1$ ($g_1$ is related to $g_2$ of with Type $i$). 

To investigate the situations of  these three types of interactions of  MDSs, we 
associate a directed graph %$G=(V, E)$
with labeled edges to each MIC contig in data $\mathcal D$ as follows.

 Each graph $G=G_M=(V(G_M), E(G_M))$, which may be disconnected and have multiple connected components, corresponds to a MIC contig $M$.
 Each vertex $g\in V(G_M)$ corresponds to a MAC contig $g$ whose MDSs are segments of the MIC contig $M$. 

A labeled directed edge $(g_1,c,g_2)$ is in  $E(G_M)$ if $c=c(g_1,g_2) \not = (0,0,0)$. 
%$g_1$ interacts $g_2$  by a combination of types described below,
%where $c$ is the color specified. 

In the figures below we use colors on the edges to indicate the labels of the edges:
red$=(1,1,1)$, green$=(1,1,0)$, blue$=(1,0,1)$,
orange$=(0,1,1)$, purple$=(1,0,0)$, cyan$=(0,1,0)$, and black$=(0,0,1)$.

%\begin{itemize}
%\item (Red) All three types.
%
%\item (Green) Types 1 and 2 but not 3.
%
%\item (Blue) Types 1 and 3 but not 2.
%
%\item (Orange) Types 2 and 3 but not 1.
%
%\item (Purple) Type 1 only.
%
%\item (Cyan) Type 2 only.
%
%\item (Black) Type 3 only.
%
%\end{itemize}

\begin{figure}[h!]
\begin{center}
\includegraphics[scale=.2]{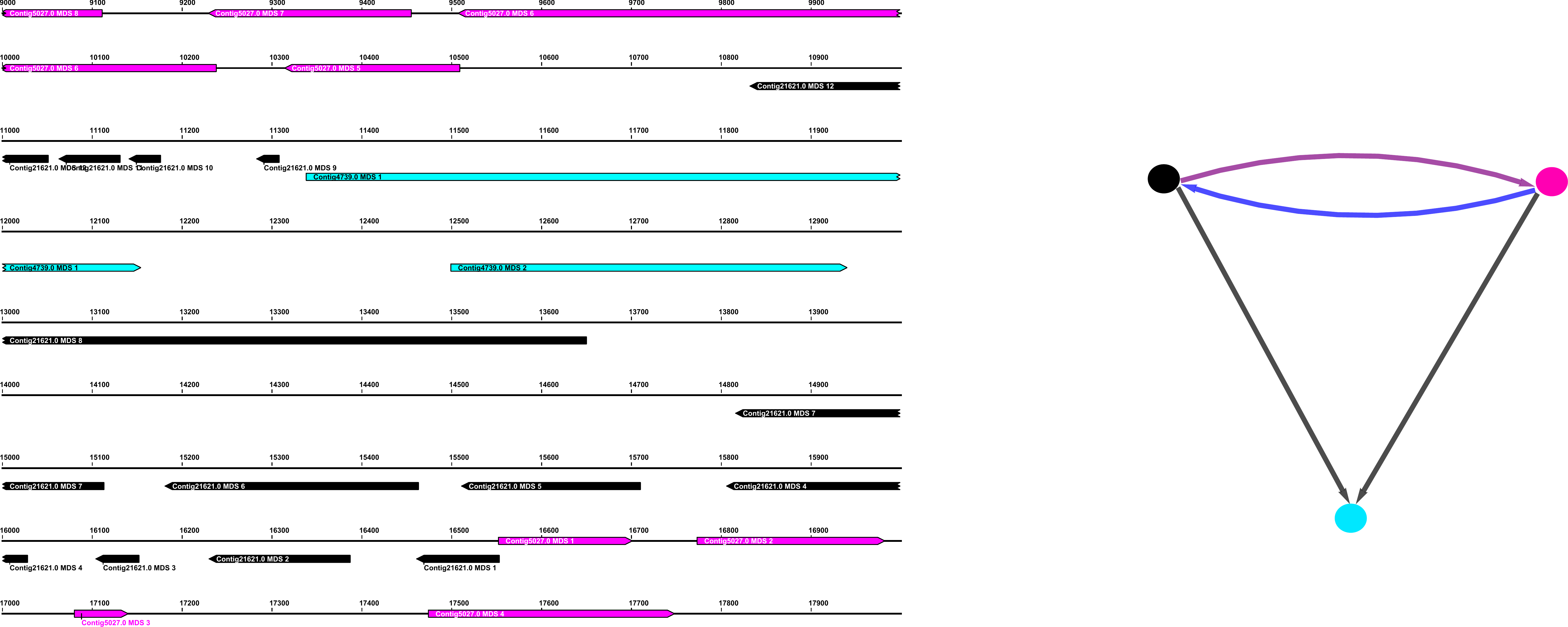}
 \put (-62,109) {\small{$(1,0,0)$}}
 \put (-62,82) {\small{$(1,0,1)$}}
 \put (-102,55) {\small{$(0,0,1)$}}
 \put (-25,55) {\small{$(0,0,1)$}}
\caption{A segment of MIC contig ctg7180000069854 (left) and its corresponding graph (right). There is a relatively short overlap of MDSs of the black  and purple contigs
and the purple contig interleaves with the  black (hence there is a blue edge indicating label $(1,0,1)$ from the purple to the black vertex) but there is no MDS segment of the 
black % purple 
contig 
that interleaves with the %interleaving with 
purple contig, so there is a purple edge 
indicating $(1,0,0)$ in the opposite direction. IESs of both black and purple contigs contain MDSs of the cyan contig, so there are two black edges indicating labels $(0,0,1)$ ending at the cyan vertex.}
\label{segment}
\end{center}
\end{figure}

Figure~\ref{segment}
shows a locus of the MIC contig 
ctg7180000069854 containing MDSs of three MAC contigs 5027.0 (purple), 
21621.0 (black), and 
4739.0 (cyan). 
Figures~\ref{ctg7180000088928} 
and~\ref{ctg7180000067411} depict some 
other  examples of graphs 
for  MIC contigs in the data. 

%In \cite{chen2014architecture}, the whole genes of {\it Oxytricha trifallax} was obtained.
%The database \cite{Burns2016} contains this information.
%This data was further processed in \cite{burns2016recurring} to obtain a data ${\mathcal D}$, that contains the loci of MDS segments in MIC contig for each MAC contig. This data was used to construct a graph for each MIC contig as described above.

The set of graphs corresponding to the data $\mathcal D$ that is analyzed here is denoted $\mc G_{\mc D}$ such that 
$\mc{G}_{\mc D}=\{ G_M\mid M$ is a MIC contig in $\mc D\,\}$.
There are 629 distinct colored graph isomorphism classes and 288 isomorphism classes of colored connected components of the graphs in $\mc D$, and they can be found at:
\url{http://knot.math.usf.edu/data/Colored_Components/index.html}.
%There are 283 distinct vectors is set $S$ obtained with this the above adjustment.

\subsection{Graph Features Selection}\label{subsec:graphfeatures}

We describe a method of converting graph data set to a set of points in the Euclidean space $\mathbb{R}^n$, i.e., the point cloud.

To each (directed and colored) graph $G$ in our data set we associate a vector $P(G) \in \mathbb{R}^n$ obtained by using relevant numerical graph invariants. This vector is then considered as a point in $\mathbb{R}^n$ (Figure~\ref{coordinate}).

\begin{figure}[h]
  \centering
   {\includegraphics[scale=0.05]{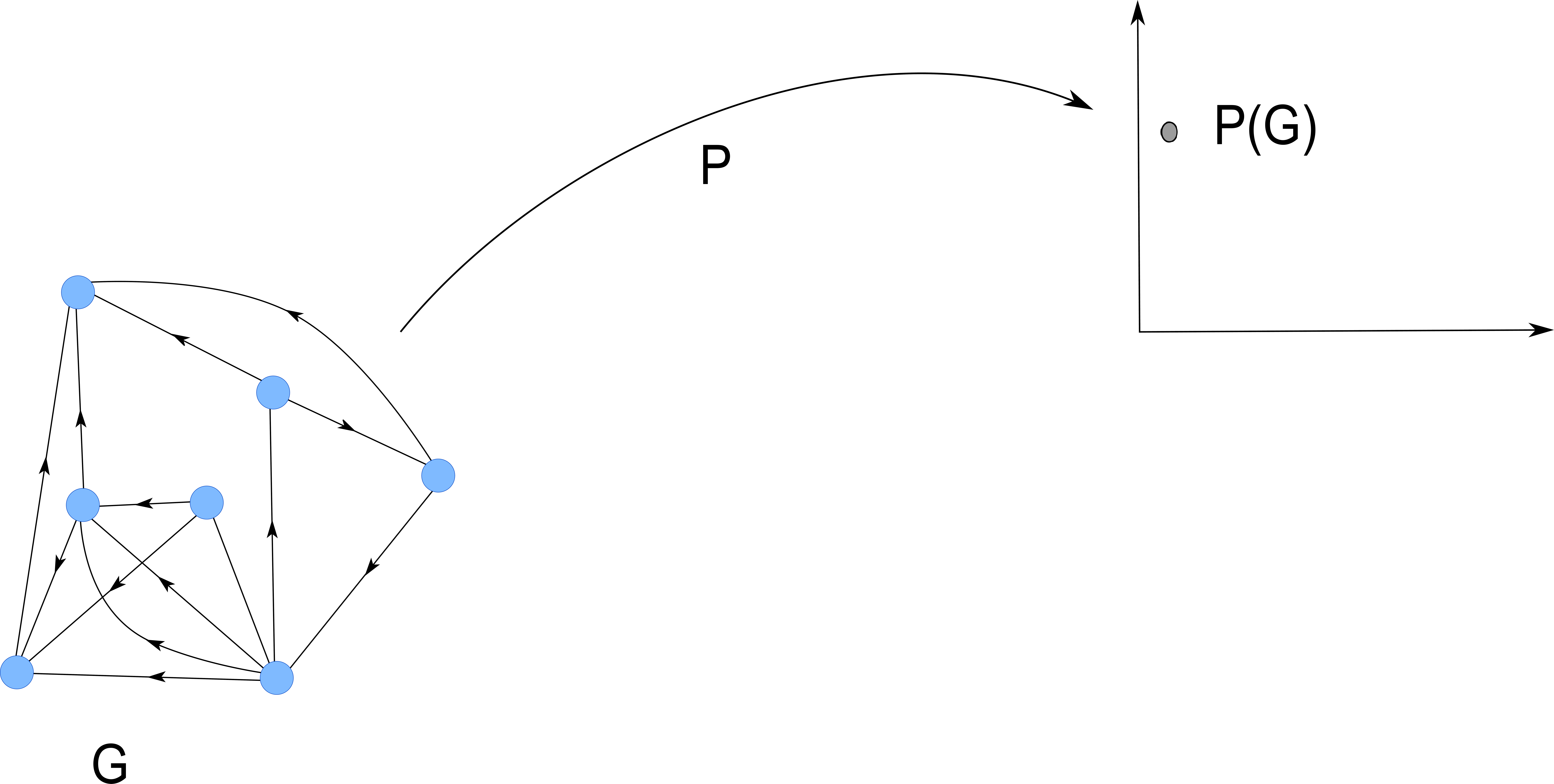}
     \put (-10,120) {$\mathbb{R}^n$}
   \caption{Every graph $G$ is associated to a feature vector $P(G)$, a point in the Euclidean space $\mathbb{R}^n$.}
   \label{coordinate}
  }
\end{figure}

The vector $P(G)$ is  
obtained by using local, vertex specific, and global, graph specific, features of $G$. In this first global analysis of the genome we cluster the data according to general graph structure properties, therefore for each $G \in \mathcal{G}_{\mathcal{D}}$ we also consider a corresponding undirected graph $U(G)$. The undirected, uncolored graph $U(G)$ is obtained from $G$ by replacing each pair of parallel edges with opposite directions in $G$ with an undirected edge, and 
by ignoring the direction and the colors of the edges as shown in Figure~\ref{example}.
\begin{figure}[h]
  \centering
   {\includegraphics[scale=0.07]{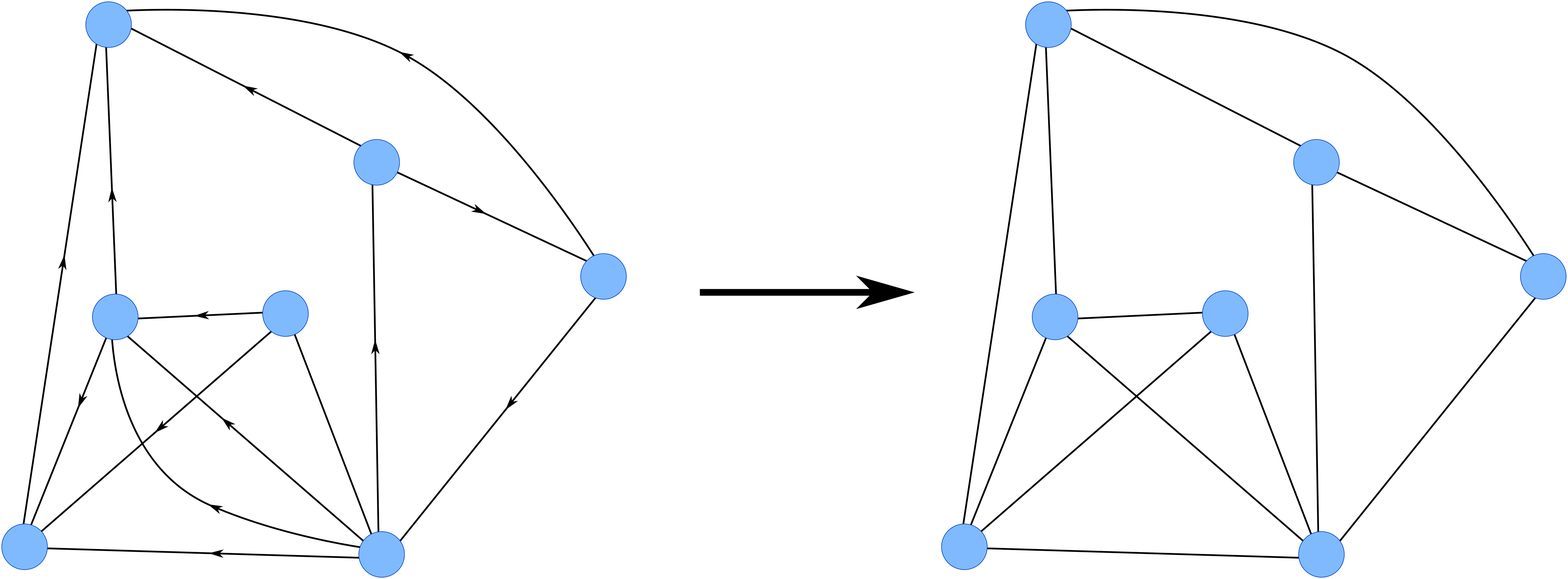}
   \caption{An undirected graph associated to a directed one.}
   \label{example}
  }
\end{figure}

\noindent {\bf Global Vector.} A vector $P_{gl}(G)$ with three entries, called the {\em global vector}, is associated to each graph $G \in \mathcal{G}_{\mathcal{D}}$. This vector $P_{gl}(G)$
consists of three features $\langle |V(G)|,|E(G)|,CN(G) \rangle$ where  $|V(G)|$ and $|E(G)|$ are the numbers of vertices and edges in $G$, respectively,  and 
$CN(G)$ is the size of the largest clique in $U(G)$. The isolated vertices are not counted in $|V(G)|$ as they represent MAC contigs that have no  interrelation with any other MAC contig present in the MIC contig represented by the graph. In our data the maximum number of vertices is 43, the maximum number of edges is 74, and the largest clique size is 6 (appears twice in the data).

\noindent {\bf Local Vector.} Vectors that use local properties of the vertices are associated to each $G\in \mathcal{G}_{\mathcal{D}} $. %These vectors use the properties of the vertices that appear locally in the graph. 
 For each vertex $v_i$ we consider two numbers,  its valency $val(v_i)$, and the clique number $cq(v_i)$. 
 The valency $val(v_i)$ is a summation of its out-degree and its in-degree 
(including the parallel oppositely oriented edges) and the number of cliques of given sizes. 
The clique number $cq(v_i)$ is the number of cliques (induced subgraphs of $U(G)$ isomorphic to the complete graph $K_k$ for some $k$) that contain this vertex. 

The vertices in $G$, $v_1,v_2,\ldots, v_{|V(G)|}$ are ordered such that their valences are non-increasing. Vertices that have the same valency are further ordered such that their clique numbers are non-increasing. 
This order remains fixed for the graph $G$.
The \textit{valency vector}, denoted $P_{val}(G)$, consists of a list of
valencies of the preordered vertices $P_{val}(G)=$$\langle {\rm valence}(v_i)\rangle_{v_i \in V(G)}$ of the graph $G$. The maximum valency of a vertex in our data is $29$ achieved by contig 67157 with 25 outgoing edges and 4 incoming edges. The 23 of the outgoing edges have $(0,0,1)$, one edge has label $(0,1,0)$ and the other is labeled $(0,1,1)$. The maximum outgoing valency is 25 and it is achieved by the contig 67157. The maximum incoming valency is 6 and it is achieved by contig 67223.

%in a non-increasing  order (see  Figure~\ref{order}). 
% We use the notation  to refer to the valence vector of a graph $G$. 
 
% The vertices are ordered in such a way their valencies are non-increasing. Define the valence vector   

\begin{figure}[h]
  \centering
   {\includegraphics[scale=0.07]{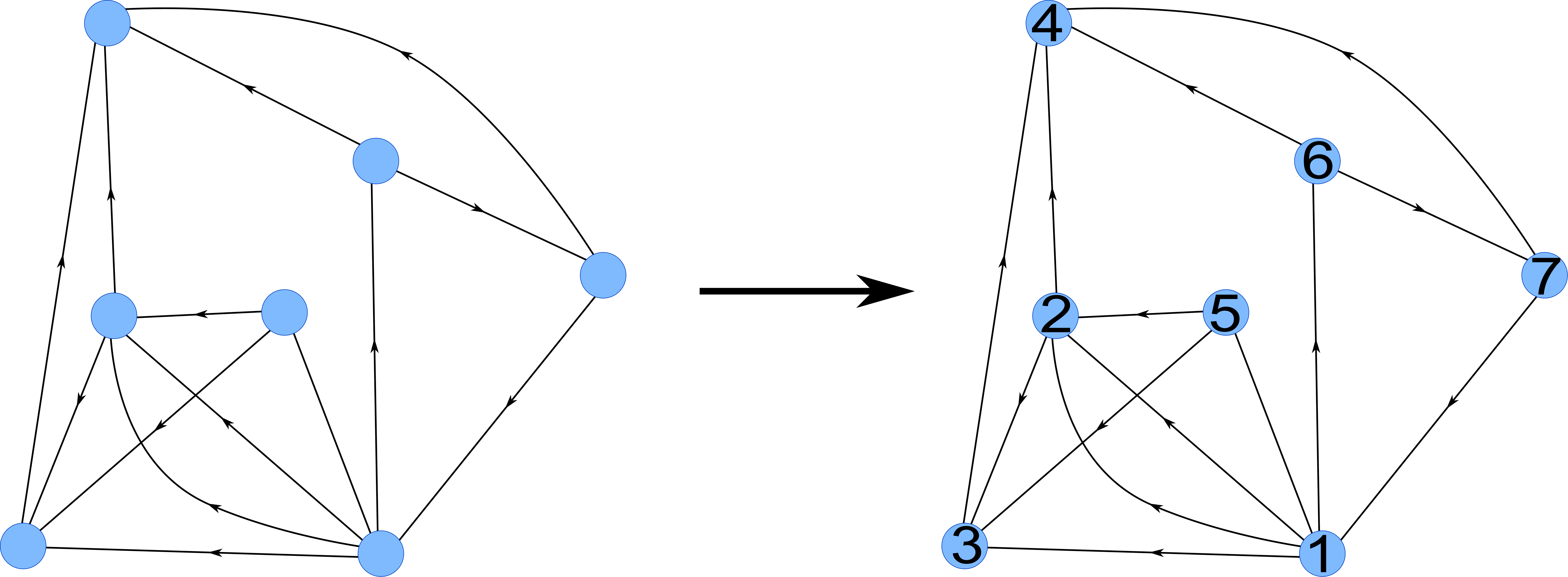}
   \caption{Vertices of a graph $G$ are ordered by their valances. Here, $P_{val}(G)=\langle 6,5,4,4,3,3,3 \rangle $.}
   \label{order}
  }
\end{figure}

%\noindent {\bf The Clique Vector.} 
%The third vector is obtained by looking at cliques associated to $U(G)$. An $n$-clique is a complete subgraph of $n$-vertices.
%For every vertex $v_i$ in the graph $G$ we associate the number of cliques (induced subgraphs of $U(G)$ isomorphic to the complete graph $K_k$ for some $k$) that contain this vertex. 
 The vertex order of the clique vector follows  the same  predetermined order of vertices for $G$. 
 %defined for $P_{val}(G)$. 
 We denote this vector by %$P_{cq}(G)=\langle$\# cliques $(v_i)\rangle_{v_i\in V(G)}$.
 $P_{cq}(G)=\langle cq(v_i)\rangle_{v_i\in V(G)}$.
An example of construction of $P_{cq}(G)$ is depicted in  Figure~\ref{clique_vector}.

\begin{figure}[h]
  \centering
   {\includegraphics[scale=0.06]{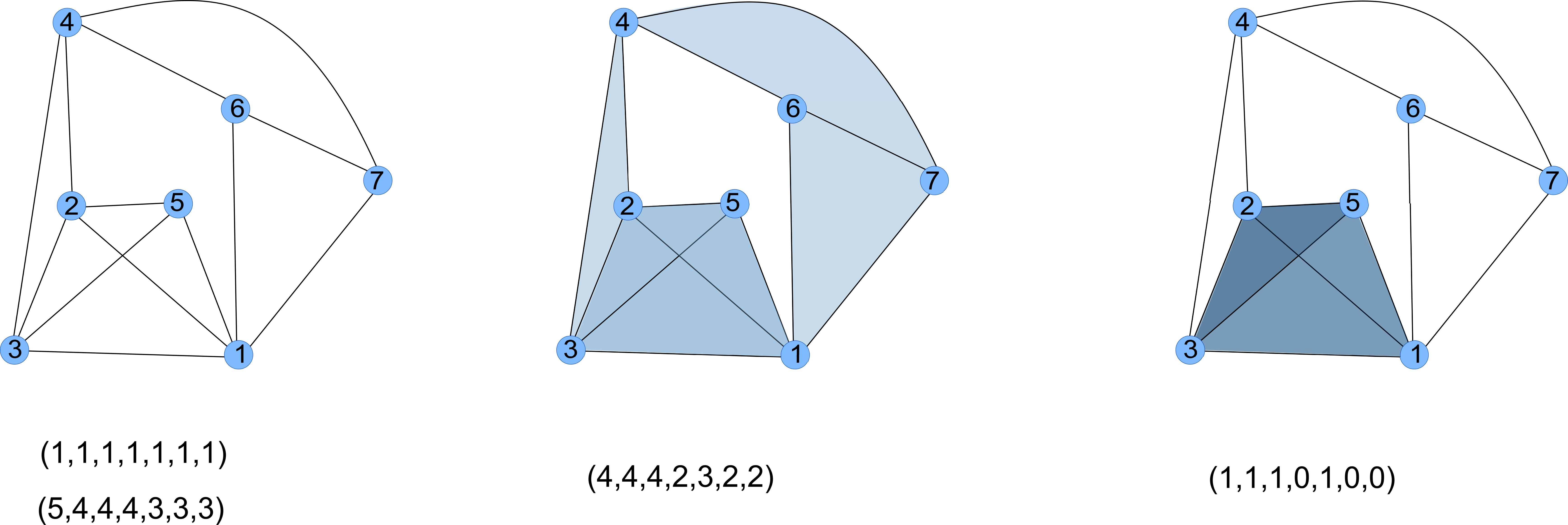}
   \caption{The number of cliques associated with the vertex $v_i$, vertices ordered as in Figure~\ref{order}. Left: $k=1$ vertices and $k=2$ edges. Middle: $k=2$. Right: $k=3$. This graph has no 
   cliques of size higher than $4$. 
   The clique vector in this example is $P_{cq}(G)=\langle 11,10,10,7,8,6,6 \rangle$ which is the sum of the $4$ vectors.}
   \label{clique_vector}
  }
\end{figure}

\noindent {\bf The Graph Vector.} 
The {\it graph feature vector} $P(G)$ is defined by concatenating the vectors $P_{gl}(G)$,  $P_{val}(G)$ and $P_{cq}(G)$. For a graph $G$, the number of entries of the vectors $P_{val}(G)$ and $P_{cq}(G)$ are the same
 and therefore %.Let $|V(G)|$ be the number of vertices of $G$, then 
$P(G)$ is a vector in $\R^{2|V(G)|+3}$. We denote the set of vectors associated to 
$\mc{G}_{\mc D}$ with 
$S_{\mc D}$, or simply $S$ .

%\subsection{Normalization} \label{subsec:normalize}

%Following the suggestion from Nathan Lazar, we noramlize 
%entries of feature vectors. 
%Let $\mathcal{G}=[G_1,...,G_n]$ be the given graph data set.
%Let $P(G_i)=\langle v^i_1, \ldots, v^i_{k(i)} \rangle$, where $k=|V(G_i)|$. 
%Recall that each entry $v^i_j$ is non-negative for all $i$ and $j$.
%Denote by $M_j$ and $m_j$ the maximum and minimum, respectively, of $j$-th entry of $P(G_i)$ over all $i$:
%$$ M_j={\rm max} \{ \ v^i_j \ | \ i=1, \ldots, n \ \} , \quad
%m_j = {\rm min} \{ \ v^i_j \ | \ i=1, \ldots, n \ \}. $$
%Then redefine each entry $v^i_j$ by $(v^i_j - m_j) / (M_j - m_j)$, which we call
%{\it (affine) normalized entries}. 
%We use the same notation $v^i_j$ for the new normalized entries.
%Note that the new entries satisfy $0 \leq v^i_j \leq 1$ for all $i, j$.
%We retain the same notation for $P(G_i)$ for the features vector with normalized entries.

\noindent {\bf The Point Cloud.}
Observe that the number of entries of the vectors in $S$ % $S_{\mc D}$ 
is not uniform, because this number depends on the number of vertices in the corresponding graph. 
In order to work in the common Euclidean space, we expand some of the vectors (by appending 0's) to obtain a consistent number of entries in all vectors. This modification is obtained in the following way. 

%Let $\mathcal{G}=\{G_1,\ldots,G_n\}$ be a graph data set. We map the set $\mathcal{G}$ to a Euclidean space in order to obtain clustering by  TDA analysis. We start by applying the coordinate function $P$ constructed above to every graph in $\mathcal{G}$. The point $P(G_i)$ 
%in general belong to different Euclidean spaces. The points need to be embedded in the same Euclidean space before doing any further analysis. For this purpose, we augment each feature vector $P(G_i)$ by zeros to make them all in the same space. To put the vectors $P(G_i)$ in the same Euclidean space for all $i=1, \ldots, n$, 
%we modify the vectors $P(G_i)$ as follows.

Let $d=\max\{ \ |V(G)| \ | \  G\in \mc{G}_{\mc D} \  \}$. 
If  the valence  vector of $G$ is 
%originally defined in Subsection~\ref{subsec:normalize}
 $$
 P_{val}(G)=\langle 
v_1, v_2,\ldots, v_{|V(G)|}
\rangle, 
$$
then we construct an auxiliary valence vector for $G$ with 
%redefine $P_{val}(G_i)$ by 
$$
\hat P_{val}(G)=\langle 
v_1, v_2,\ldots, v_{|V(G)|}, 0, \ldots, 0
\rangle
$$
increasing the number of entries of $P_{val}(G)$ to $d$ such that 
 $d-|V(G)|$ entries of zeros are added at the end.
%The new clique vector $P_{cq}(G)$ is similarly defined by concatenating $d-|V(G)|$ copies of zeros to make it a vector of dimension $d$.
Similarly we construct auxiliary clique vector $\hat P_{cq}(G)$ by adding $d-|V(G)|$ zeros at the end of $P_{cq}(G)$. 
The graph vector $\hat P(G)$ is redefined with 
the concatenation $\langle P_{gl}(G), \hat P_{val}(G),\hat P_{cq}(G)\rangle$. For our graph data the $\mathcal{G}_{\mathcal{D}}$ the maximum number of vertices in a graph is $d=43$.

We abuse the  notation and use $P(G)$ instead of $\hat P(G)$ to refer to the zero-augmented feature vector associated with a graph $G$. The final point cloud set $S=\{P(G)| G\in \mc{G}_{\mc D} \,\}$ forms a subset of $\mathbb{R}^{2d+3}=\mathbb{R}^{89}$. 

For comparison, we also consider the point cloud $S_{gl}$ from the global features vector 
$P_{gl}(G)$. The point cloud $S_{gl}$ is in $\R^3$. There are $283$ vectors obtained with the above adjustments.

It is important to notice that the entries of the vectors $P$, and $P_{gl}$ for a graph $G$ are graph isomorphism invariants. 

\begin{lemma}
If the graphs $G$ and $G^{\prime}$ are graph isomorphic (but not necessarily label preserving) then $P(G)=P(G^{\prime})$ and $P_{gl}(G)=P_{gl}(G^{\prime})$.
\end{lemma}

\begin{proof}
Let  $\phi:G\rar G'$ be a graph isomorphism.  
Then $G$ and $G'$ have the same number of vertices, edges and the size of the maximal cliques
in their undirected versions $U(G)$ and $U(G')$. Therefore $P_{gl}(G)=P_{gl}(G')$ and the first three entries of 
$P(G)$ and $P(G')$ are the same. Also the number of non-zero entries in $P(G)$ and $P(G')$ are the same. 
A graph isomorphism maps vertices of $G$ to vertices of $G'$  with the same number of outgoing and incoming edges. 
Similarly, the number of cliques incident to a vertex in $U(G)$ is the same to the number of cliques of the corresponding vertex in $U(G')$. Let $V_1,\ldots, V_s$ be a partition of $V(G)$ such that 
\begin{itemize}
\item[i] for all $v,w\in V_i$, for all $i=1,\ldots,s$ $val(v)=val(w)$ and $cq(v)=cq(w)$, and 
\item[ii] for all $v\in V_i$ and $w\in V_j$ with $i<j$ $(i,j=1,\ldots,s)$, 
either $val(v)>val(v')$ or $val(v)=val(v')$ and $cq(v)>cq(v')$.
\end{itemize}
Then $\{ \phi(V_1),\ldots, \phi(V_s) \}$ is a partition of the vertices of $G'$ satisfying the properties [i] and [ii].
Any order of vertices of $V(G)$ (resp. $V(G')$) that has non-increasing valencies and non-increasing clique numbers must list
vertices of $V_i$ (resp. $\phi(V_i)$) before vertices of $V_j$ (resp. $V(G')$) whenever $i<j$.  Therefore it must be
that 
$P_{val}(G)=P_{val}(G')$ and $P_{cq}(G)=P_{cq}(G')$.
\end{proof}

In our analysis there are three reasons that reduced the data from $688$ graphs to $283$. Many of these graphs are isomorphic if the edge color is ignored, and directed graphs often reduce to isomorphic undirected graphs.
Of course there are graphs $G$ and $G'$ that are non isomorphic but $P(G)=P(G')$.  Consider attaching two edges to a 4-cycle to obtain  a 6-vertex graph. They can be attached to neighboring or to diagonally opposite vertices of the cycle. In both cases the associated vectors will be the same. 

%%%%%%%%%%%%%%%%%%%%%%%%%%%%%%%%%%%%%%%
\section{Clustering Analysis with TDA} 
%%%%%%%%%%%%%%%%%%%%%%%%%%%%%%%%%%%%%%%
For  a data set $S \subset \mathbb{R}^n$, in our case corresponding to a set of directed graphs, a TDA analysis 
of persistent $0$-dimensional homology gives rise to a
hierarchy of connected components of (clustered) graphs as described below.

To understand the distribution of the points of $S$ in $\mathbb{R}^{n}$ we 
use the notion of the neighborhood graph, as defined below,  to construct a hierarchy of undirected graphs whose vertices are $S$. The neighborhood graph of $S$ depends on a chosen distance function. In our case the distance $d$ is the Euclidean distance 
between two points 
$\sqrt{\sum_i(x_i-y_i)^2}$.

%The neighbohood graph of a point cloud $S$ for a non-negative number $\epsilon \geq 0$ is defined as follows.

\begin{definition}
Let $S$ be a set of points in $\mathbb{R}^n$ and let $\epsilon\geq 0$ be a non-negative number. The {\em $\epsilon$-neighborhood graph} is an undirected graph $G_{\epsilon}(S)$, where $G_{\epsilon}(S) = (S, E_{\epsilon}(S))$
and
$E_{\epsilon}(S) = \{[u, v] \mid d(u, v) \leq  \epsilon, u, v \in  S, u\neq v\}$.
\end{definition}

%The $\epsilon$-neighborhood graph
%can be used to obtain a hierarchical clustering on a point cloud data set. This
The clustering analysis  is done by considering a %nested 
sequence of neighborhood graphs $G_{\epsilon_1}(S), G_{\epsilon_2}(S),\ldots$ 
for $S\subset \mathbb{R}^n $ obtained by a sequence of incrementally increasing values 
$\epsilon_1<\epsilon_2<\cdots$. 

\begin{definition}
A {\em cluster} of $S$ at level $\epsilon$ is a connected component in the neighborhood graph $G_\epsilon(S)$.
\end{definition}

%\subsection{Properties of Graph Vectors}
We observe some facts about the graphs vectors $P$ and $P_{gl}$. % that will aid us in our cluster analysis.
Suppose $\mc G$ is a family of graphs and $S=S(\mc G)$ and $S_{gl}=S_{gl}(\mc G)$ are points in $\mb R^n$
obtained as described above. The vectors in the set $S$ and $S_{gl}$ are all part of the integer lattice of $\mb R^{n}$ 
and $\mb R^3$ respectively,
therefore any distance between two distinct vectors is at least $1$. The observation below indicates that small changes in the graphs represented by the vectors induce larger distances of the vectors in $S$.

\begin{lemma}\label{gaps}
Let $G,G'\in \mc G$. Then the following hold.
\begin{itemize}
\item[(a)] If $G^{\prime}$ is obtained from $G$ by addition of one vertex and one edge incident to that vertex. Then $d(P(G),P(G^{\prime}))\geq 3$ and $d(P_{gl}(G),P_{gl}(G^{\prime}))\geq \sqrt{2}$.
\item[(b)] If $G^{\prime}$ is obtained from $G$ by addition of one directed edge without changing the total number of vertices, nor the number of cliques, then $d(P(G),P(G^{\prime}))\geq 2$. % and $d(P_{gl}(G),P_{gl}(G^{\prime}))\geq 1$.
\item[(c)] If  $G^{\prime}$ is obtained from $G$ by addition of one edge that adds a clique to the graph $U(G^{\prime})$
without changing the number of vertices, then $d(P(G),P(G^{\prime}))\geq \sqrt{5}$.
\item[(d)] If $G'$ is obtained from $G$ by changing the target of one edge from vertex $v$ to vertex $v'$ without changing the number of the cliques, either $d(P(G),P(G^{\prime}))\geq \sqrt{2}$ or
$d(P(G),P(G^{\prime}))=0$.
\end{itemize}
\end{lemma}

\begin{proof}
(a)
The addition of a vertex in $G'$ changes the number of non-zero entries in $P(G')$ in two places, once at $P_{val}(G')$
and again at $P_{cq}(G')$. Let $w$ be the new vertex in $G'$ added to $V(G)$ and let $[v,w]$ be the new edge in $G'$ connecting $v\in V(G)$ with the new vertex $w$. Then $w$ can be taken to be the last vertex in $V(G')$ in the order of the vertices, while the order of $v$ in $V(G')$ might be either the same as its order in $V(G)$ or different. 
In both cases the entries in $P(G')$ corresponding to 
$|V(G')|$, $|E(G')|$, $val(v)$, $cq(v)$, $val(w)$  are at least one more than the corresponding entries  in $P(G)$ and the entry of $cq(w)$ is at least two more (a 1-clique vertex $w$ and a 2-clique the new edge) than the corresponding entry in $P(G)$ which is 0. So $d(P(G),P(G'))=\sqrt{\sum_i(x_i-y_i)^2}\ge\sqrt{5+2^2}\ge 3$, and $d(P_{gl}(G),P_{gl}(G'))\ge\sqrt{2}$.

The proofs of (b) and (c) follow a similar argument. Note that in case of (b), if the new directed edge is incident to vertices $v$ and $w$, then because the number of cliques in $U(G')$ is not changed from the number of cliques in $U(G)$, there is an edge in $G$ incident to $v$ and $w$ in opposite direction. So $P(G')$ has at least one more in the entries $|E(G)|$, 
$val(v)$ and $val(w)$. Observe that this may imply a change in the order of the vertices, in which case there may be a difference in the entries corresponding to the $cq(v)$ and $cq(w)$ which would increase the distance between the vectors. Therefore $d(P(G),P(G'))\ge \sqrt{3}$.
%and  $d(P_{gl}(G),P_{gl}(G'))\ge \sqrt{1}=1$.

For the case of (c), the entries of $|E(G')|$,  $val(v)$, $val(w)$, $cq(v)$, $cq(w)$ in vector $P(G')$ have a change of at least one and therefore the distance $d(P(G),P(G'))\ge \sqrt{5}$ and  $d(P_{gl}(G),P_{gl}(G'))\ge \sqrt{1}=1$. The case (d) follows the argument of (b) if the valencies change or, if valencies don't change, the graphs are represented by the same vectors and the distance is 0.
\end{proof}

\subsection{Analyzing The Data Using Neighborhood Graphs}
 
A \textit{filtration of a graph} $G$ is a sequence of nested graphs $  G_1 \subseteq G_2 \subseteq \cdots \subseteq G_k=G$ where each $G_i$ is a subgraph of $G_{i+1}$. 
The definition of the neighborhood graph on a point cloud $S$ naturally induces a filtration 
for a connected graph with vertices  $S$. Namely, 
 given a point cloud $S \in \mathbb{R}^n $ and a finite sequence of non-negative numbers $0=\epsilon_1 < \epsilon_2 < \cdots < \epsilon _k $ we obtain a filtration  $G_{\epsilon_1}(S) \subseteq G_{\epsilon_2}(S) \subseteq \cdots\subseteq G_{\epsilon_k}(S)$. We assume that $\epsilon_1=0$, which implies that $E(G_{\epsilon_1})=\emptyset$. 
This filtration  also helps to extract  the connected components (clusters) of $S$ at various spatial resolutions. For a given $\epsilon$, each connected component of $G_{\epsilon}(S)$ corresponds to a cluster of directed graphs whose corresponding points in $\mathbb{R}^n$ are connected by edges that are of lengths less than $\epsilon$. This means that the vectors associated with the graphs are at 
most $\epsilon$ apart, i.e., the graph properties indicated in the vectors are similar and are within $\epsilon$ neighborhood from each other. To have a better information about the topological properties encoded in a filtration one usually considers \textit{the persistence diagram} of the filtration. For our purpose, the persistence diagram 
describes a way  the connected components of the neighborhood graph merge together as we increase the value of $\epsilon$. The persistence diagram is also equivalently described by the persistence \emph{barcode}~\cite{Ghrist2008}.  The barcode  construction is described as follows.

Let $S=S_{\mc D} \subset \mathbb{R}^n$, where $n=2d+3$ (in our case $n=89$).
In  Figures~\ref{barcode} and \ref{secondbarcode}, 
the vertical axis enumerates points of $S$, and  $\epsilon$-values are listed on the horizontal axis. 
At $\epsilon_1=0$, $E(G_{\epsilon_1})=\emptyset$, and each point of $S \subset \R^n$ forms a single connected component. There are $|S|$ connected components, and hence
the number of bars in the barcode at value  $0$ is equal to the number of data points in $S$ corresponding to
 the birth of all connected components.
%As $\epsilon$ increases the number of connected components of $G_{\epsilon}(S)$ is reduced. 
With 
appropriate increments 
of $\epsilon$ new edges are added to the neighborhood graph and the connected components start joining each other 
forming larger clusters. The merging event  of  %between two 
connected components is represented by a
 termination of all but one of the corresponding bars of the barcode.
 The choice of the bar that does not terminate in a merge of components is arbitrary,
 and we use the established convention  (see ~\cite{Ghrist2008}) where bars are vertically ordered by their length 
 from the shortest at the bottom of the diagram to the longest on the top.

The number of connected components of the graph $G_{\epsilon}$ is the number of horizontal bars intersecting the vertical line at distance equal to $\epsilon$. For instance, from Figure \ref{barcode} we deduce that the number of connected components in $G_{\epsilon}(S)$ is $2$ for $\epsilon=15$ indicating two clusters at that distance. 
Typically, the filtration ends with a neighborhood graph that has a single connected component. That is, the sequence of  $\epsilon$ values  increase from $0$ to the %$\epsilon$ 
value that gives rise to a single component graph. In the case of data $\mc D$ %and the two vectors that we studied, 
 for the set of global vectors and the point clouds $S$ and $S_{gl}$, the $\epsilon$ values range from $0$ to $22$ and $0$ to $15$ respectively.
% (i.e., a clique). 
% \end{itemize}

% Figures~\ref{barcode} and \ref{secondbarcode}
% are persistent diagrams for 
%the point cloud $S$ % $S_{\mc D}$ 
%for the  graph  vectors (including the global and the local features of the graph) and 
%the point cloud $S_{gl}$ obtained from the  global features vectors, respectively. 
% The barcode diagrams  given in Figures \ref{barcode} and \ref{secondbarcode} were generated by using the software JavaPlex \cite{Javaplex}.

\begin{figure}[H]
  \centering
   {\includegraphics[scale=0.25
   ]{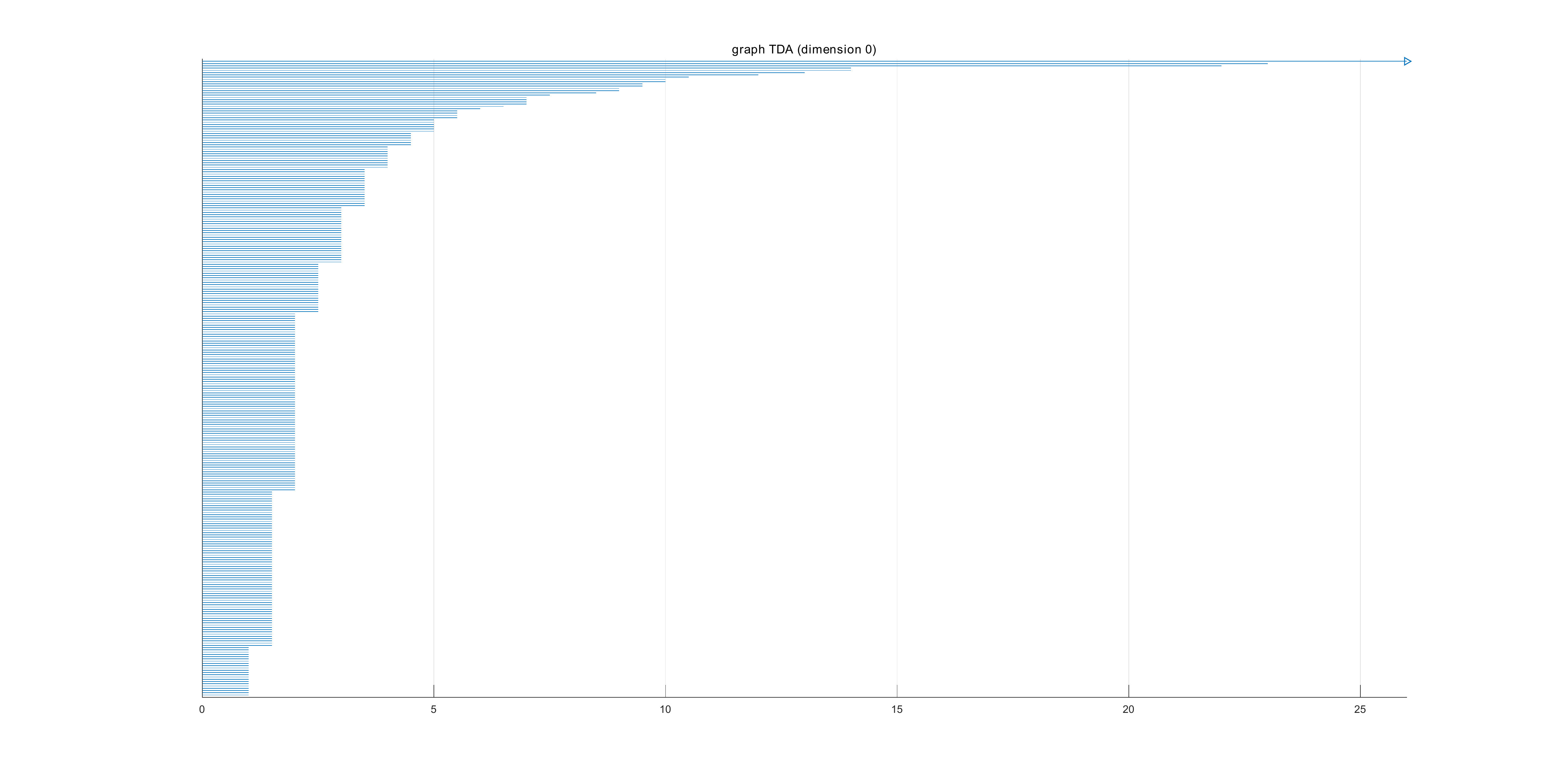}
   \caption{The barcode diagram describing the birth and death of the connected components of the neighborhood graph of the dataset $S$.}
   \label{barcode}
  }
\end{figure}

\subsection{Tree Diagrams Representing Merging Components}
The merging events of connected components described in the persistence diagram can be encoded using a tree diagram called a \textit{dendrogram}~\cite{murtagh1983survey}. 
%In this case connected components are considered as clustering and 0-persistence homology is equivalent to hierarchal clustering. 
The bottom points of the tree diagram correspond to the points of $S$ (resp. $S_{gl}$),
that  also correspond to the connected components of $G_0(S)$.
The vertical direction of the tree diagram represents values of $\epsilon$.
%When some edges from the bottom merge into a single point (an upside branch of the tree) to a single edge in the tree diagram 
%at $\epsilon_i$, this represents that 
%the corresponding points of $S$ belong to a single connected component in $G_{\epsilon_i}(S)$. 

At each level $\epsilon$ %the vertices of 
the connected components (clusters) are enumerated and each vertex in the tree is labeled by $(i,\epsilon)$ where $i$ is an index that corresponds to the $i$th cluster %connected component 
of the graph at level $\epsilon$. At each level $\epsilon$, the number of nodes corresponds to the number of clusters of $G_{\epsilon}(S)$. For a node (vertex) $v$ at level $\epsilon_i$, the children of $v$ correspond to the clusters %connected components 
at level $\epsilon_{i-1}$ (i.e., the graph $G_{\epsilon_{i-1}}(S)$) that have joined to a single connected component represented by $v$ in $G_{\epsilon_i}$.

For a large enough value of $\epsilon_k$, $G_{\epsilon_k}(S)$ is connected, and it corresponds to the single node (root) of the tree.
The dendrograms corresponding to %information encoded in 
the persistent diagrams for $S$ \ 
and $S_{gl}$
are shown in Figures \ref{tree} and Figure \ref{global-f-tree} in the supplementary documentation, respectively.

\begin{figure}[H]
  \centering
   {\includegraphics[scale=0.04]{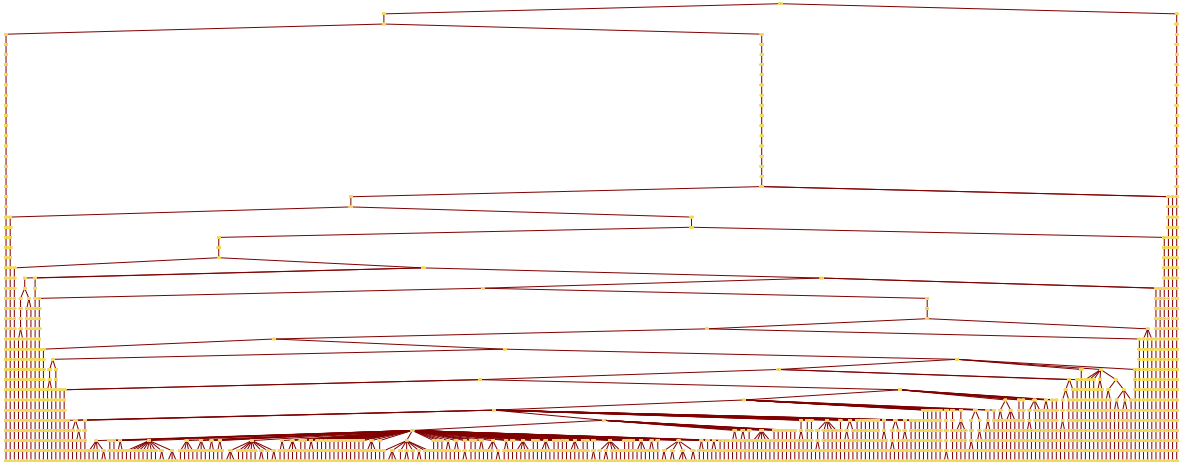}
   \caption{The dendrogram clustering tree of dataset $S$.}
   \label{tree}
  }
\end{figure}

\subsection{Implementation}
The point cloud generated from the data $\mathcal{D}$ was generated using a custom Python script. The persistence diagrams were generated using Javaplex \cite{Javaplex} and the dendrogram tree diagrams were generated using Mathematica \cite{Mathematica}. The sequence data, the graph data and the scripts are available at \url{http://knot.math.usf.edu/data/GeneSegmentInteractions/dna_graph_study/}.      

\begin{comment}
An interesting feature of our specific data set is the sharp increase of growth of in the largest connected components of the neighborhood graph. In Figure \ref{largeconnected components} we plot this relationship. The vertical access corresponds to the distance $\epsilon$ and the horizontal access correspond the number of vertices in the largest corrected component in the neighborhood graph.
\begin{figure}[H]
  \centering
   {\includegraphics[scale=0.33]{largest_component.eps}
   \caption{The growth of the largest connected component.}
   \label{largeconnected components}
  }
\end{figure}
This behavior can be visualized  by considering a \textit{multidimensional scaling} \cite{borg2005modern} planar projection of the set $S$. See Figure \ref{MDS} and note that majority of the points are going to clustering early on in in a single large components.    

\begin{figure}[h]
  \centering
   {\includegraphics[scale=0.2]{dna_mds_fig.png}
   \caption{An multidimensional scaling projection of the data set on the plane.}
   \label{MDS}
  }
\end{figure}

\end{comment}								

%%%%%%%%%%%%%%%%%%%%%
\section{Results}
%%%%%%%%%%%%%%%%%%%%

The analyzed data $\mc D$ consists of processed \cite{burns2016recurring} micronuclear contigs obtained after sequencing of {\it O. trifallax} \cite{chen2014architecture}
and is available at \cite{burns2016recurring}. The directed graphs that correspond to the contigs in $\mc D$ can be found at %\cite{Denys}. 
\url{http://knot.math.usf.edu/data/Colored_Graphs/index.html}.

As mentioned, the data $\mc D$ produced  273 distinct vector entries in $\mc S=\mc S_{\mc D}$ that correspond to the same number of isomorphism classes of graphs ranging from 2 to 43 vertices. Each MIC contig corresponds to a vector in 
$\mc S$ while the MAC contigs whose MDS segments do not have any of the types 1, 2 or 3 interactions with MDSs of other contigs represent isolated vertices in the graphs and are not taken in consideration for the construction of $S_{\mc D}$.

We constructed filtration  with $\epsilon$ increments of $.5$ in order to detect small neighborhood changes in the neighborhood graph, these sometimes are reflected by reorienting a directed edge. 

\subsection{Output of Hierarchal Clustering}

The bar code diagram and the dendrogram for the filtration and clustering of the neighborhood graph of $\mc S$ are depicted  in Figure~\ref{barcode}  and Figure~\ref{tree}. As expected by Lemma~\ref{gaps}, the neighborhood graph consists of isolated vertices for $\epsilon \le 1$ and the fist edges appear at $\epsilon =1.5$ when
 there are 14 two point and 4 
three point clusters. The two or three graphs joined at this distance differ from each other by small changes such as a
 single directed edge addition that does not change the cliques. 
 
 At $\epsilon =2$, as noted in Lemma~\ref{gaps}, most points remain distant from each other and only those representing graphs with small changes in their structure are joined by en edge. In addition two, three and in one instance four of the previously formed clusters join in (also with some additional points) to form new clusters, 
 and there are 25 new small two or three point clusters. Most of the points in $\mc S$ remain as isolated vertices. 
 
 At $\epsilon=2.5$ a dramatic change occurs and one  large cluster of 155 elements is formed with a second cluster of 5 points, and several small (two or three point) clusters. All other points stay as isolated vertices.
At this point the feature of the point-cloud becomes clear,
it consists of  a single large cluster, singletons, and some small two or three element components.

At $\epsilon=9.5$, there is one large cluster of 269 points while the second largest cluster is of 4 elements, and there are  10 isolated points.

In the last 5 digits of contig numbers, the second largest cluster consists of:
$$ 88928, \quad  88096,  \quad 67742, \quad 67187. $$
Figure \ref{4 figures} shows the graphs contained in this cluster.

%\begin{figure}[h]
%\begin{minipage}[t]{.32\textwidth}
%  \includegraphics[width=\textwidth]{ctg7180000088096}
%    \caption{ctg7180000088096}
    %\label{ctg7180000088096_STANDALONE}
 % \end{minipage}
%\end{figure} 

\begin{figure}[h]
  \centering
   {\includegraphics[scale=0.5]{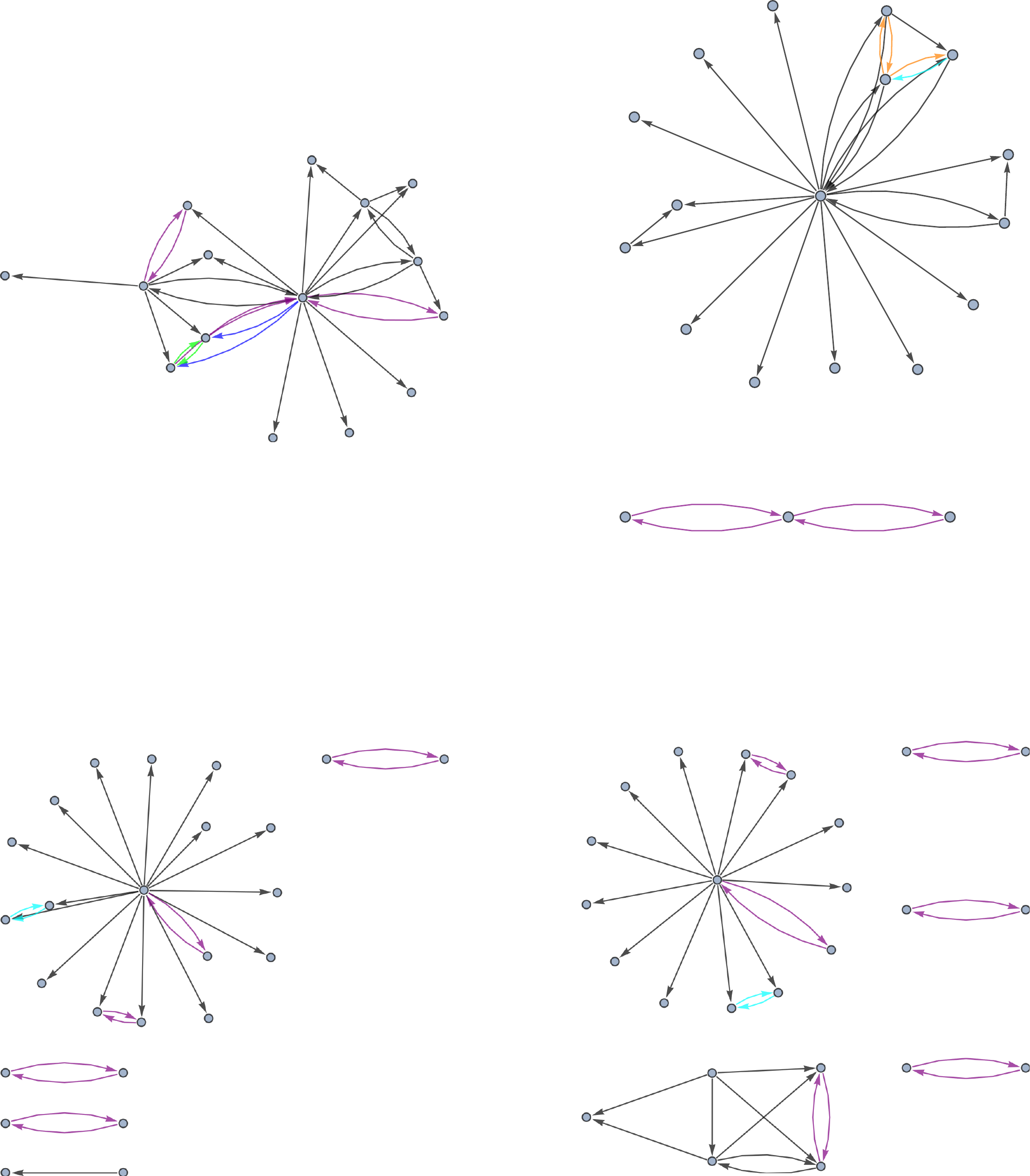}
   \caption{Top ctg7180000088928 and ctg7180000088096. Down ctg7180000067742 and ctg7180000067187.}
   \label{4 figures}
  }
\end{figure}

All four of these graphs contain a `star' vertex that is of high valency having multiple black (label $(0,0,1)$) outgoing edges. 
This indicates that there is one MAC contig whose MDSs  interleave with MDSs of multiple other MAC contigs, and we 
observed that in most of these cases it is one IES of the central `star' MAC contig that contains most or all of the MDSs of the other contigs. This coincides with the observation in \cite{g3} where the depth of these embeddings were considered. 

The isolated points belong to 10 contigs 
$$
67761, \quad
87162, \quad
87484, \quad
67363, \quad
67280, \quad
67243, \quad
67157, \quad
67223, \quad
67417, \quad 
67411.
$$
These are depicted in the corresponding figures  in Supplementary Material Section. We note that some of these graphs have  multiple `star' vertices, or the component that contains a `star' vertex also has additional cycles and cliques. In particular, the two graphs with 6-cliques (contigs $6742$ and $67223$) and the one with a 5-clique (contig $67411$)
are part of these isolated points. Furthermore, the graph with the longest path of 5 vertices (contig $87484$) indicated in 
\cite{g3} as one of the most in-depth embedding of genes within a single IES is also on this list. In all these cases we observe that the majority of the edges are black and purple, meaning that the prevailing inter-gene MDS organization  is interleaving. 

As $\epsilon$ increases, the four-element cluster becomes part of the large cluster at $\epsilon=10.5$ and the 
isolated singleton points join the large cluster one or two at the time until $\epsilon=14.5$ 
when the two, most distant contigs $67517$ and  $67223$ remain isolated until $\epsilon=22$ and $\epsilon=23$ respectively.

%for $S$ (corresponding to the all graph features vector) show that as 
%the filtering diameter $\epsilon$ increases by increment of $\epsilon=0.5$, the following clusters appear.

%These are depicted in the corresponding figures in Supplementary Material Section.
% This is the same as the first draft. MS 2017-06-21

The pattern of clusters for $S_{gl}$ is similar to that of $S$. A large single cluster is formed at value $\epsilon=1.5$, with 2 clusters of 5 elements, 3 clusters of 2 elements, and 23 singleton clusters.

At $\epsilon=4.5$, 
the clusters consist of a large single cluster, the second largest of 9 elements, the two clusters of two elements, and 5 singleton clusters.
The size two clusters are
$\{ 67417, % singleton in S
67243 % singleton in S
\}$ and 
$\{ 67187, % appears in 4-elmt cluster of S
67228 % doesn't appear in S
\}$. 
The elements of the former cluster appear as isolated points in the neighborhood $\epsilon=9.5$ of $S$, 
while $67187$ of the latter cluster,  appears in the 4-elements cluster of $S$, and 
$67228$ is in the largest cluster of $S$. 

The isolated points for $\epsilon=4.5$ are
$
67223, \, % singleton in S
67363, \, % singleton in S
67157, \, % singleton in S
67280, \, % singleton in S
87484.  % singleton in S
$
We note that all also appear as isolated points of the neighborhood graph of $S$ for $\epsilon=9.5$. 

In this case, as in the case of $S$, the two most distant graphs correspond to the contigs $67517$ and  $67223$
 that join the large cluster at $\epsilon =14.5$ and $\epsilon=18.5$ respectfully.

%The tree diagrams in Figure~\ref{tree}  for $S_
%{gl}$ (corresponding to the global features vector) show that as 
%the filtering diameter $\epsilon$ increases in the increment of $0.5$, the following clusters appear.

\begin{figure}[H]
  \centering
   {\includegraphics[scale=0.09]{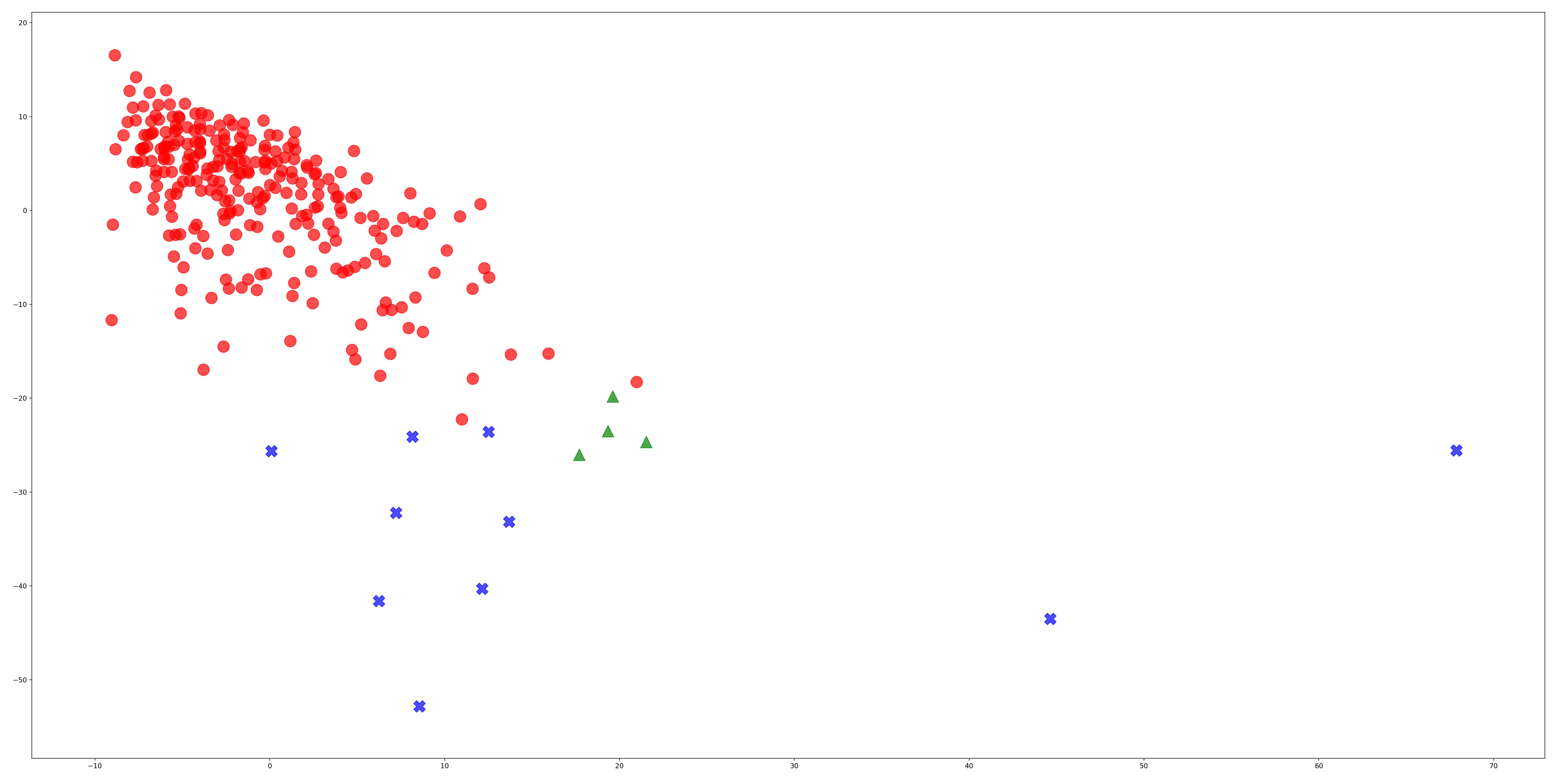}
   \caption{A 2d multidimensional scaling projection for $S$. The points of $S$ are colored according to clustering at $\epsilon=9.5$. At this level we have $12$ clusters, the largest cluster is colored red in the figure, the second largest cluster consists of $4$ elements and is colored green and the singletons are all colored blue.}
   \label{MDS_S}
  }
\end{figure}

Figures \ref{MDS_S} and \ref{MDS_Sg} 
represent the 2d multidimensional scaling projections\cite{kruskal1978multidimensional} 
of the point clouds $S$ and $S_{gl}$, respectively.

\section{Discussion}

In this paper we initiated a mathematical method of representing and analyzing gene segment relationship in a scrambled genome of 
 {\it Oxytricha trifallax} and up to our knowledge, such genome wide study for inter-gene segment arrangement has not been done before. The inter-gene segment arrangements are represented by graphs representing their segment relationship.
 We analyzed the graph data by converting these graphs to a point cloud in a higher dimensional Euclidean space and applied clustering methods from topological data analysis to identify patterns in the graph structures. 
 
The big majority of interactions within a single MIC contig are represented with small graphs up to five vertices (corresponding to the large cluster at $\epsilon=2.5$) and one can `move' from one graph to another by small vertex/edge changes. This suggests that genes with complex interaction patterns are unique and often not found among macronuclear genes.
The majority of the inter-gene segment arrangement within micronuclear chromosomes involve two or three genes and are pairwise interleaving and sometimes overlapping. 

The most prevalent multi-gene segment arrangements in the {\it Oxytricha}'s genome are interleaving and often this appears as one 
gene (or one IES in a MAC contig) interleaving with multiple other MAC contigs (as seen through the `star' like vertices). In the  cluster consisting of four graphs, a single IES of the `star' MAC contig interleaves
with multiple MAC contigs. All star contigs are scrambled, which 
follows the analysis in \cite{g3} where it was observed that contigs whose IESs interleave with other MAC contigs are mostly scrambled. 

The graph representation of the inter-gene segment relationship introduced here is novel, and we hope that a similar approach can be used in studies of the scrambled genomes of other species. 
Comparisons among orthologous genes in other species with scrambled genomes may reveal whether patterns in these graph structures are conserved or embellished over evolutionary time.
Furthermore, studies whether genes with interleaved gene segments are co-expressed may indicate whether the rearrangement of these MAC segments are in parallel or sequential. We suggest that models of gene rearrangement should also focus on operations that can be applied to these frequent interleaving gene segments, which in some cases resemble the odd-even patterns detected within scrambled genes \cite{burns2016recurring}.

The construction of the graph data into a point cloud in this paper is by a vector whose entries are common graph invariant properties, such as the number of vertices,  edges and cliques.
We used two vectors, one that had more local vertex properties and the other in $\mb R^3$ which included only the number of vertices, edges and the maximal clique. It is interesting that in both cases the isolated points are the same, and much distant from the rest of the points. The rearrangement process of the MIC contigs corresponding to these isolated points may indicate specific biological process that include multiple genes simultaneously. The graphs with large cliques (5 and 6) imply that segments of up to 5 or 6 genes all mutually interleave and we suggest further rearrangement gene analysis for these situations. 
In our study we did not consider the length of overlapping segments, nor the number of interleaving gene segments. Further methods that include edge weights may be suitable for  more detailed analysis.

Our representation of graphs relied mainly on representation of a graph via a feature vector in $\mathbb{R}^n$. Similar attempts in this direction have been made for  graph similarities~\cite{gibert2012graph,riesen2010graph}. Their work focus on undirected graphs and do not consider the local properties that we used. 
  There are other venues that rely on developing a similarity measure between graph \cite{bunke2011recent,papadimitriou2010web,hajij2017visual} or graph kernels \cite{gartner2003graph,baur2005network} that we have not explored here. These are methods that often relay on the structural properties of the graph sometimes identified through topological methods and they may also reveal other properties in the genome.  For example, such methods have been successfully applied in protein function prediction~\cite{borgwardt2005protein} and chemical informatics \cite{ralaivola2005graph}. Comparison of such graph analysis methods will be subject of another study.

\begin{comment}
%%% Not clear what was intended: %%%
From our perspective the structure that the neighborhood graph provides is critical for our analysis as global perspective of data is difficult using the method mentioned above. However, exploring data from other perspectives remains of a value of our work and we plan to investigate this further in the future.  
\end{comment}

\bibliography{network_tda_vis}

\newpage
\setcounter{page}{1}
\section*{Supplementary Material}

%Persistent diagram, tree diagram, and MDS for $S$.

%\begin{figure}[H]
%  \centering
%   {\includegraphics[scale=0.04]{MDS_allfeatures_normalized_withoutCC_1.png}
%   \caption{MDS for $S$.}
%   \label{MDS_S}
%  }
%\end{figure}

%MDS_allfeatures_normalzied_withoutCC_1

%Persistent diagram, tree diagram, and MDS for $S_{gl}$.

\begin{figure}[H]
  \centering
   {\includegraphics[scale=0.8]{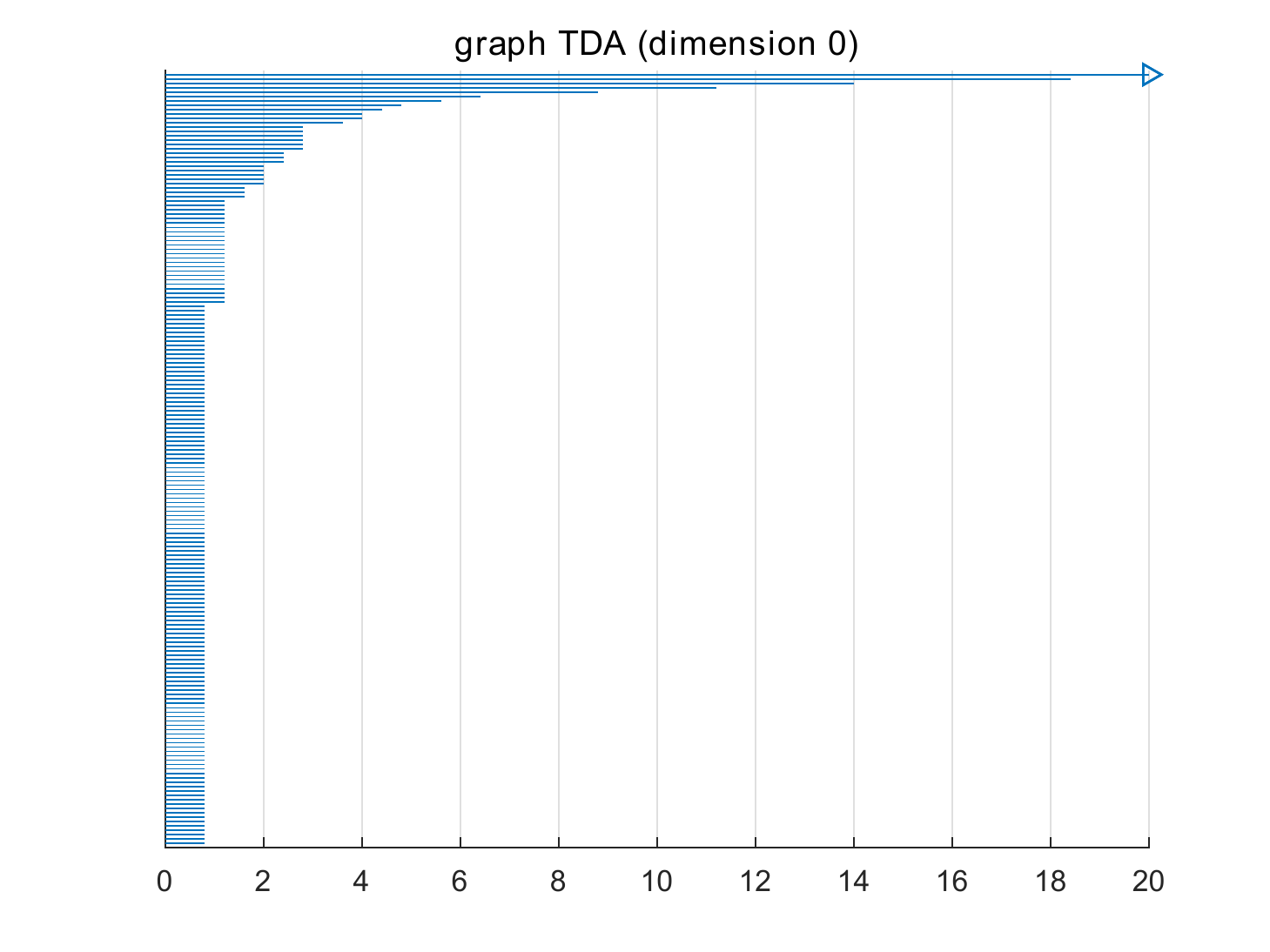}
   \caption{The persistence diagram of the data set $S_{gl}$ obtained from the global features.}
   \label{secondbarcode}
  }
\end{figure}

%Furthermore, the clustering tree is shown in Figure \ref{global f tree}.

\begin{figure}[H]
  \centering
   {\includegraphics[scale=0.039]{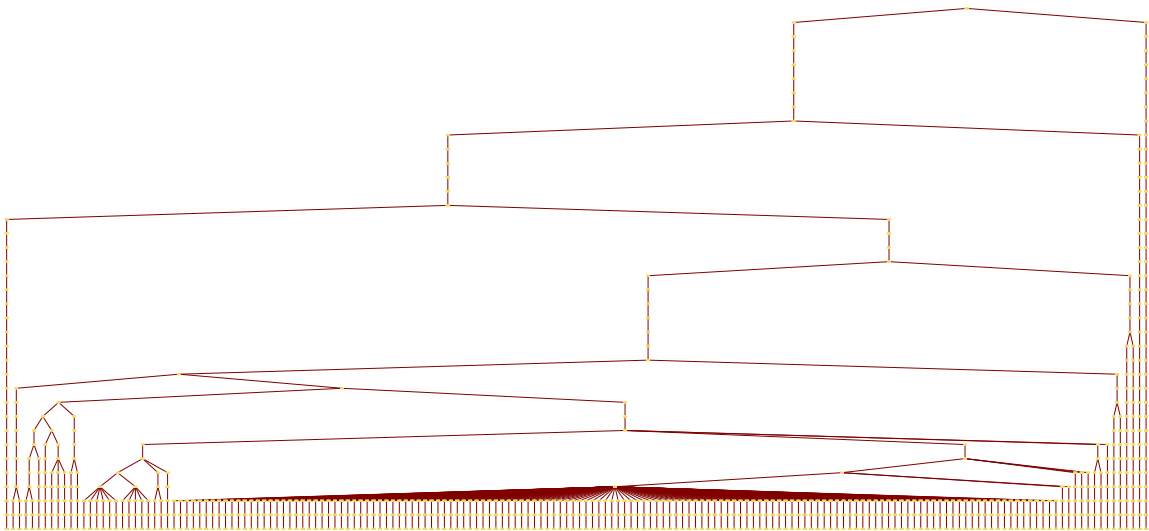}
   \caption{The clustering tree of the data $S_{gl}$ set obtained from the global features.}
   \label{global-f-tree}
  }
\end{figure}

\begin{figure}[H]
  \centering
   {\includegraphics[scale=0.09]{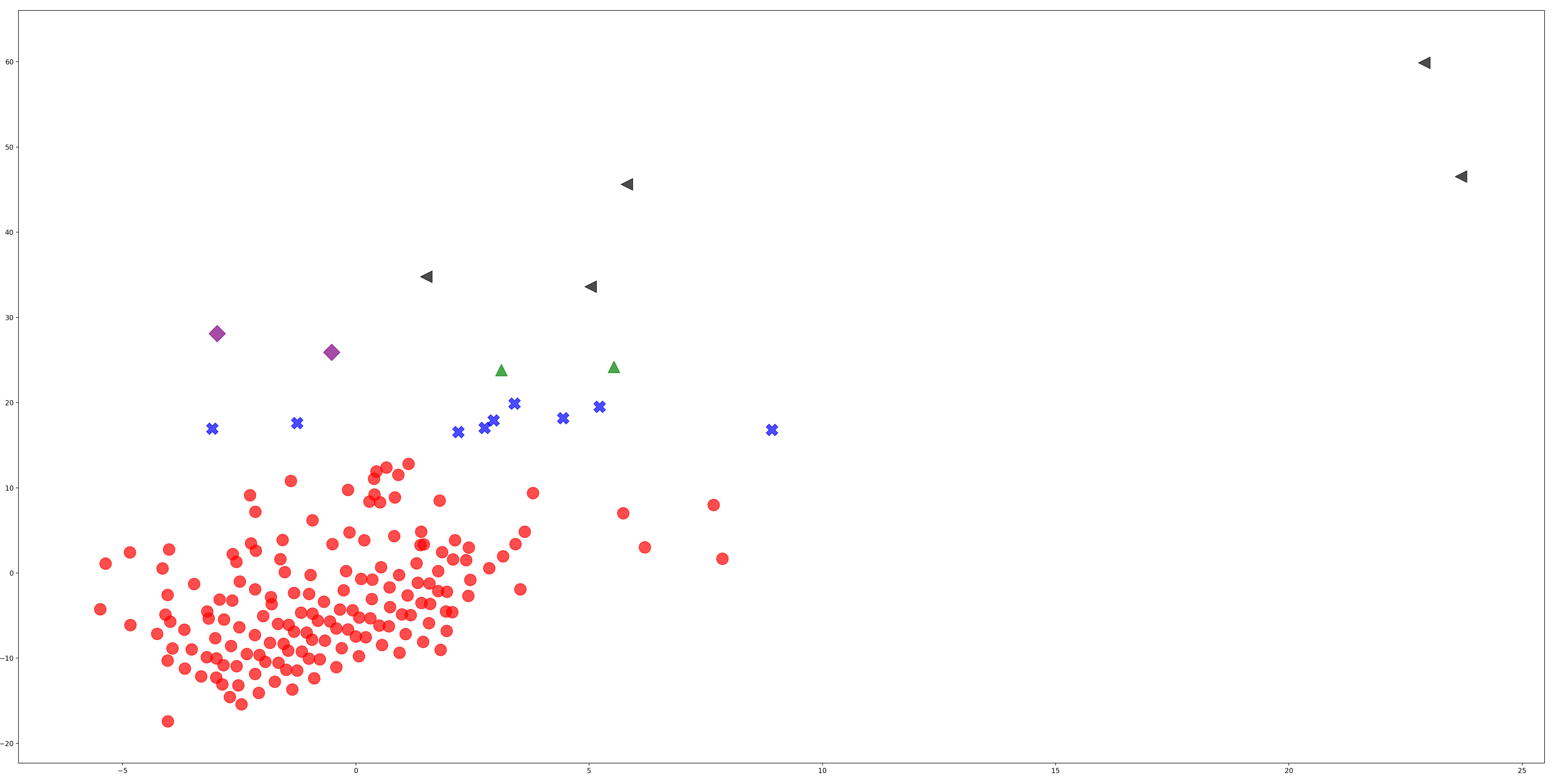}
   \caption{A 2d multidimensional scaling projection for $S_{gl}$. The points of $S_{gl}$ are colored according to clustering at $\epsilon=4.5$. At this level we have $9$ clusters, the largest cluster is colored red, the second two largest clusters consists of $2$ elements and they are colored magenta and green, and the singletons are all colored black.}
   \label{MDS_Sg}
  }
\end{figure}

The  graphs in the cluster of $S$  at $\epsilon=9.5$ consisting of four elements depicted in 
Figures~\ref{ctg7180000088928},\ref{ctg7180000088096},\ref{ctg7180000067742} and~\ref{ctg7180000067187}.

\begin{figure}[h]
\begin{minipage}[t]{.32\textwidth}
    \includegraphics[width=\textwidth]{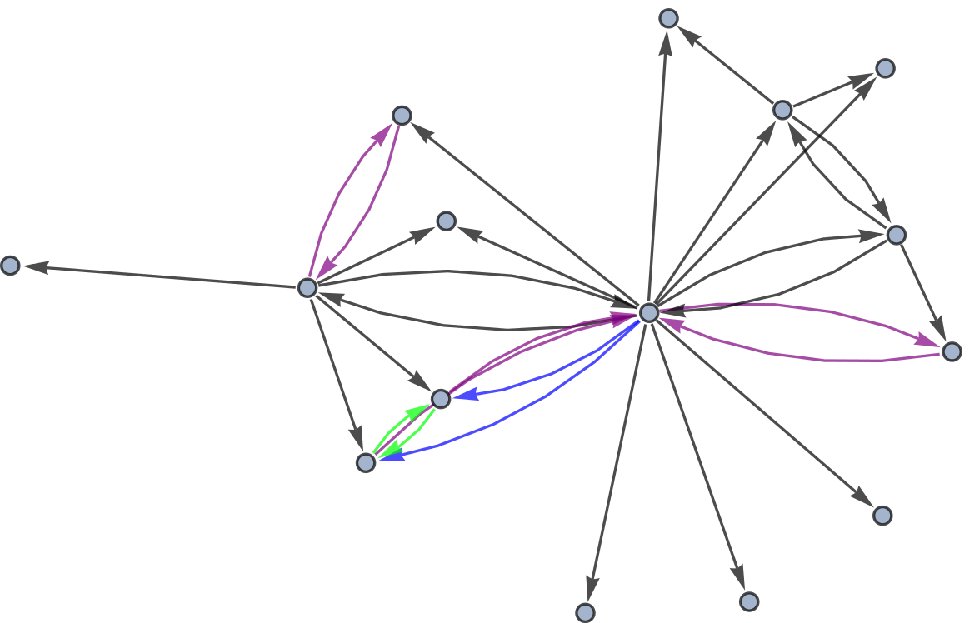}
    \caption{ctg7180000088928}
    \label{ctg7180000088928}
  \end{minipage}
   \hspace{20mm}
  \begin{minipage}[t]{.32\textwidth}
    \includegraphics[width=\textwidth]{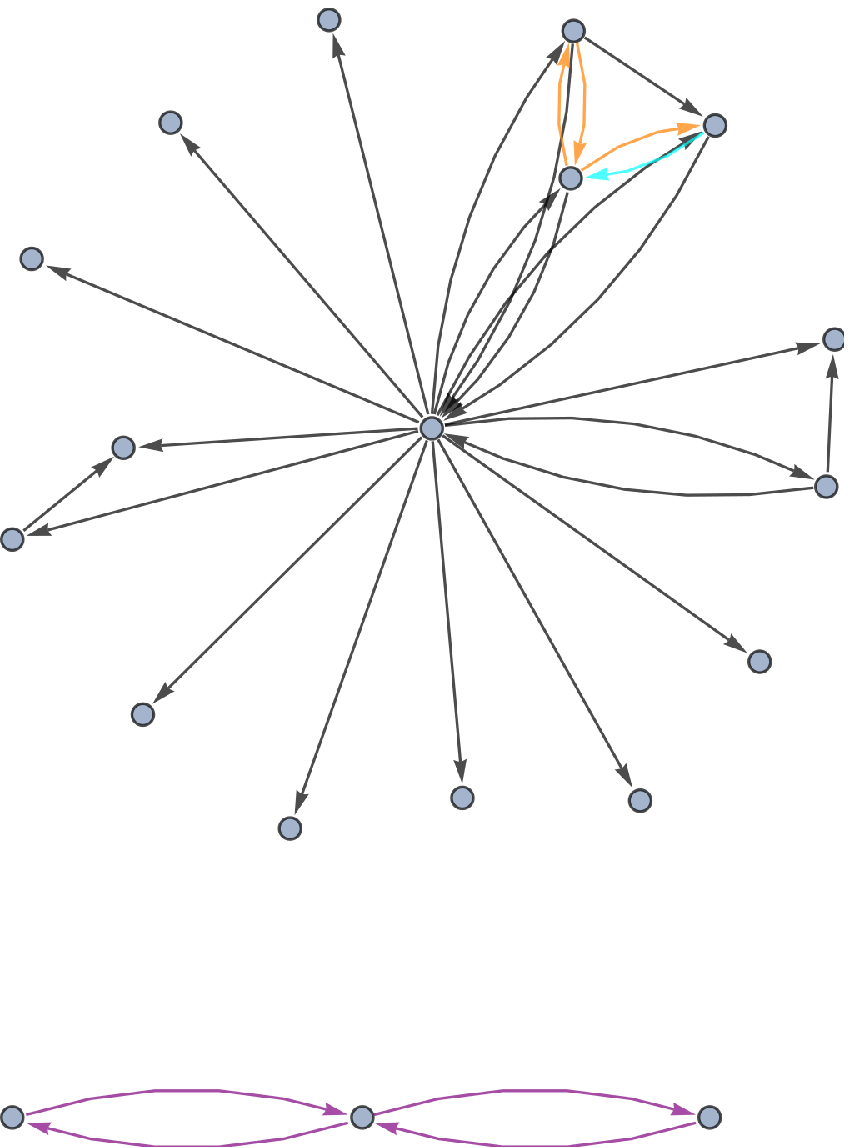}
    \caption{ctg7180000088096}
    \label{ctg7180000088096}
  \end{minipage}
\end{figure}  

\begin{figure}[h]
  \begin{minipage}[t]{.32\textwidth}
    \includegraphics[width=\textwidth]{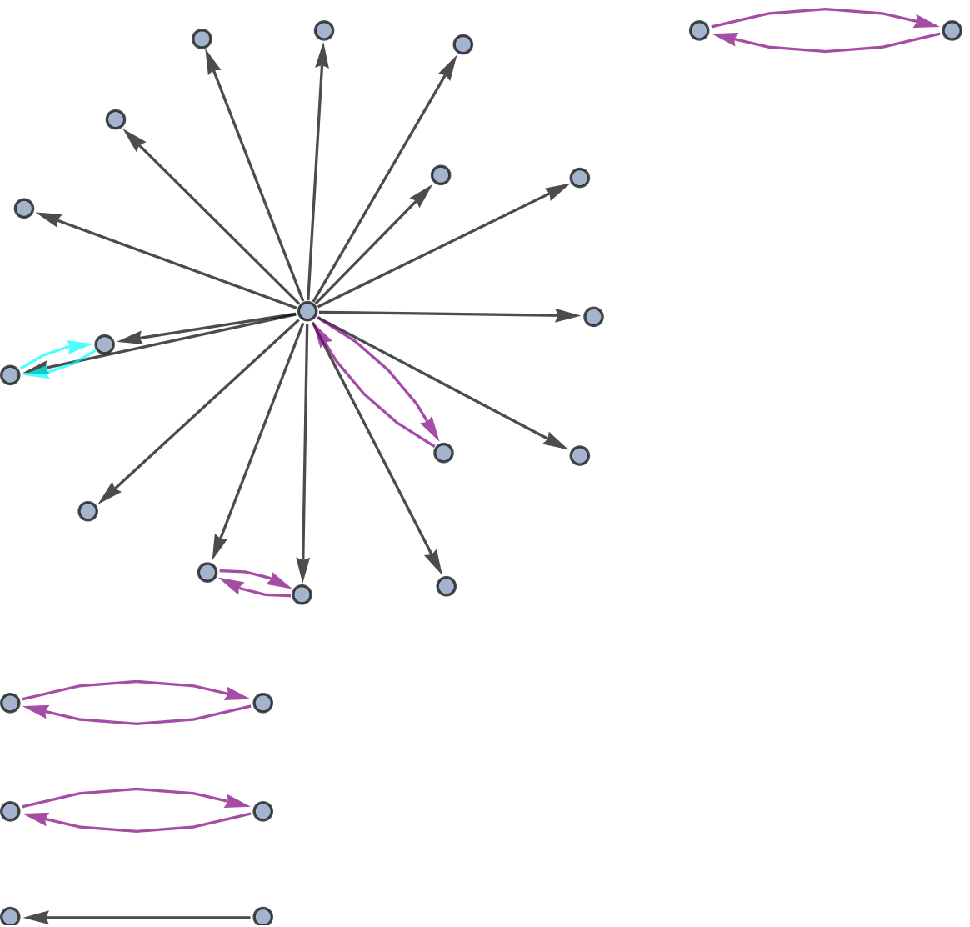}
    \caption{ctg7180000067742}
     \label{ctg7180000067742}
  \end{minipage}
  \hspace{20mm}
  \begin{minipage}[t]{.32\textwidth}
    \includegraphics[width=\textwidth]{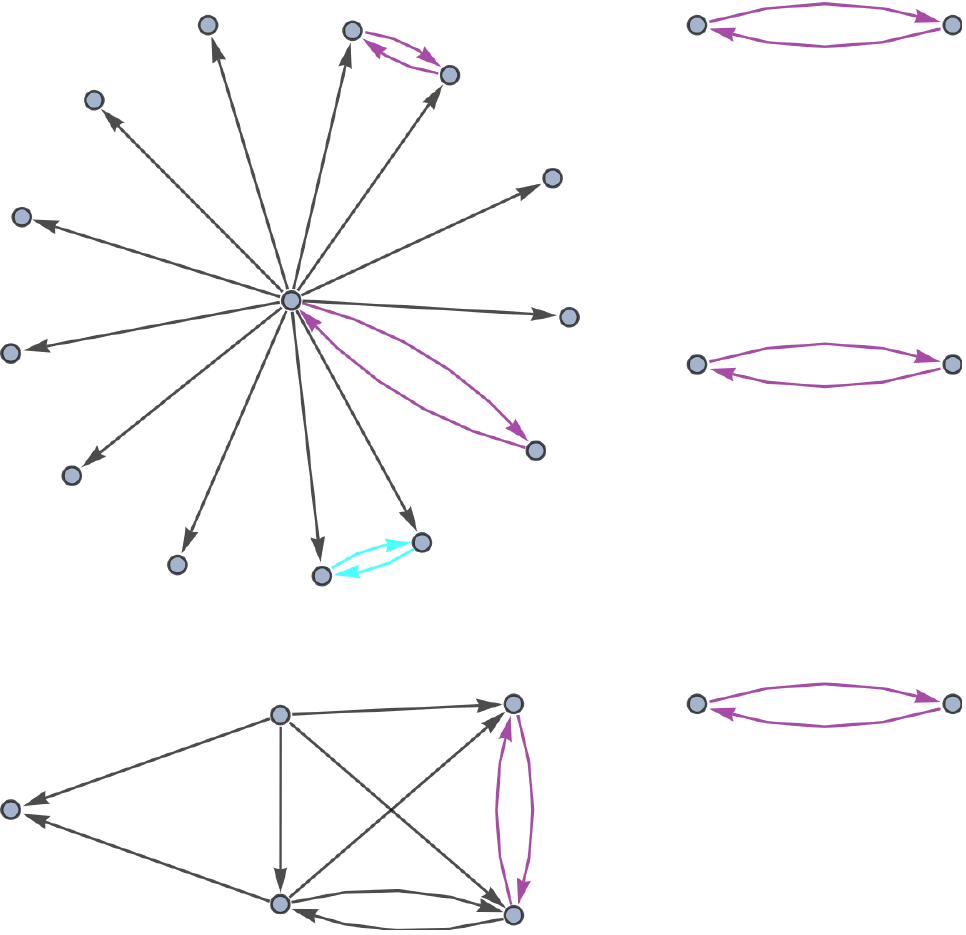}
    \caption{ctg7180000067187}
    \label{ctg7180000067187}
  \end{minipage}
\end{figure}

\pagebreak

The graphs that form singleton clusters (isolated points) in $S$ ($\epsilon=9.5$) are
depicted in the remaining figures. 

\begin{figure}[H]
  \begin{minipage}[t]{.35\textwidth}
    \includegraphics[width=\textwidth]{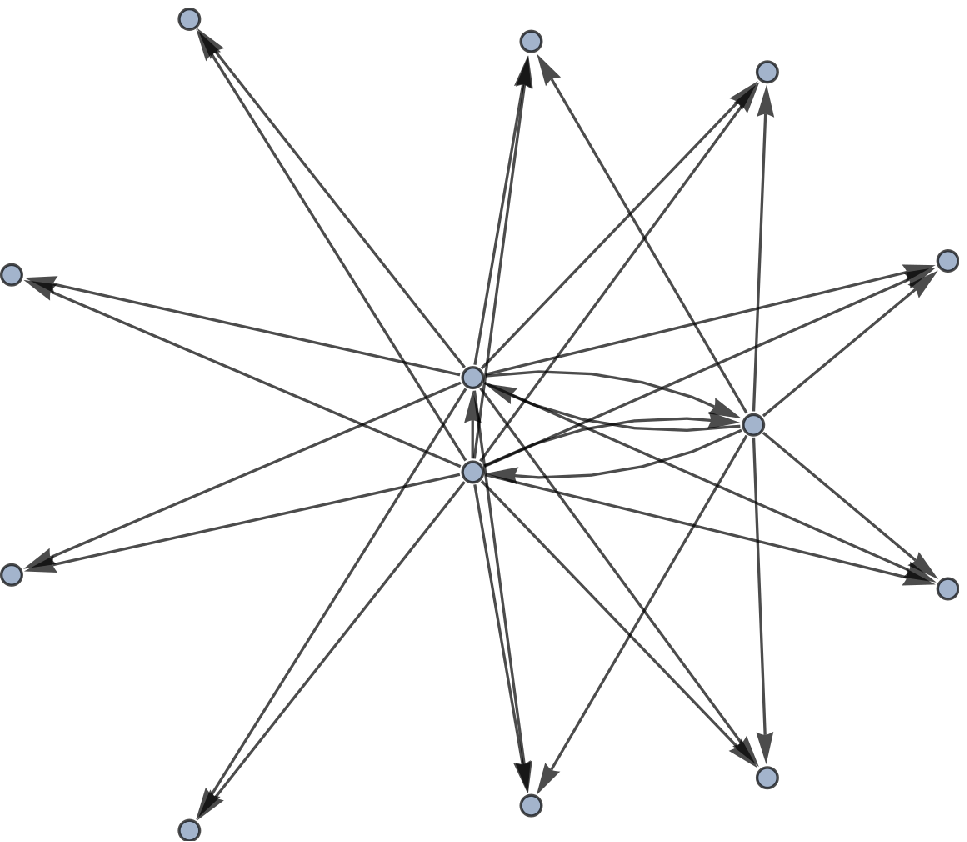}
    \caption{ctg7180000067761}
    %joins the large component at $d=11$ }
     \label{ctg7180000067761}
  \end{minipage}
  \hspace{20mm}
  \begin{minipage}[t]{.35\textwidth}
    \includegraphics[width=\textwidth]{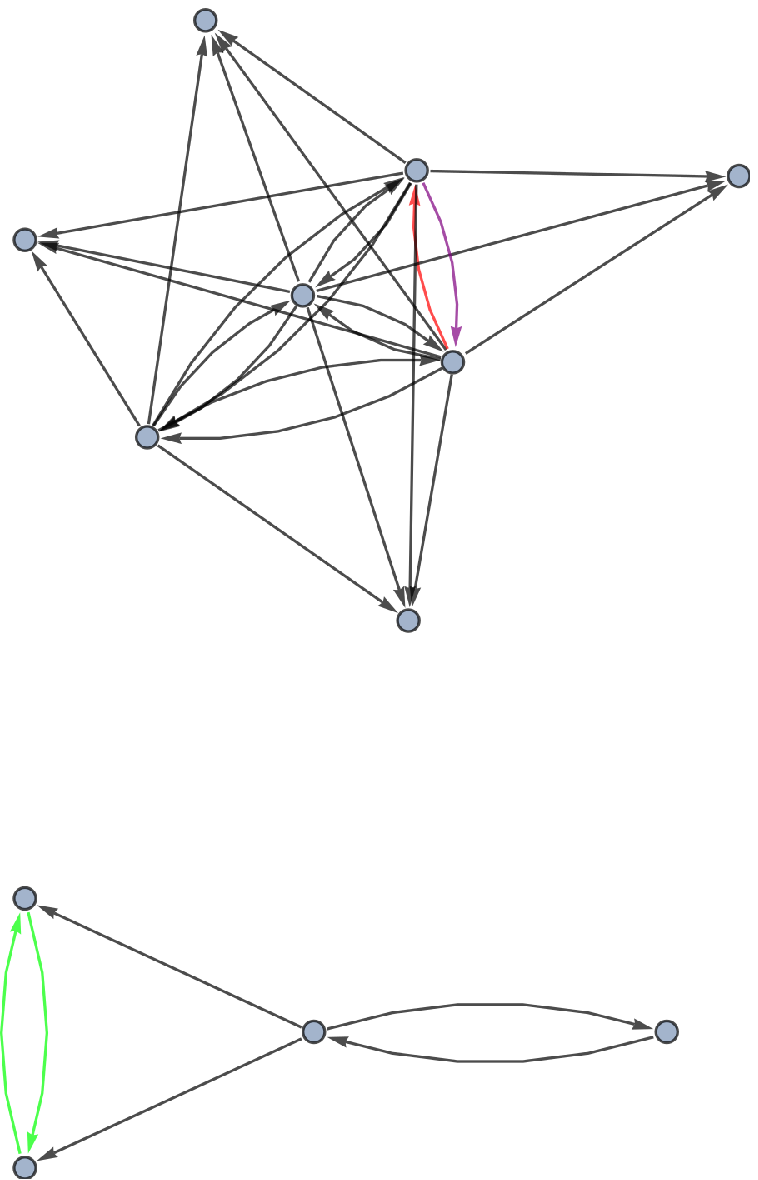}
    \caption{ctg7180000087162}
    %joins the large component at $d=12$}
    \label{ctg7180000087162}
  \end{minipage} 
\end{figure}

\begin{figure}[H]
 \begin{minipage}[t]{.32\textwidth}
    \includegraphics[width=\textwidth]{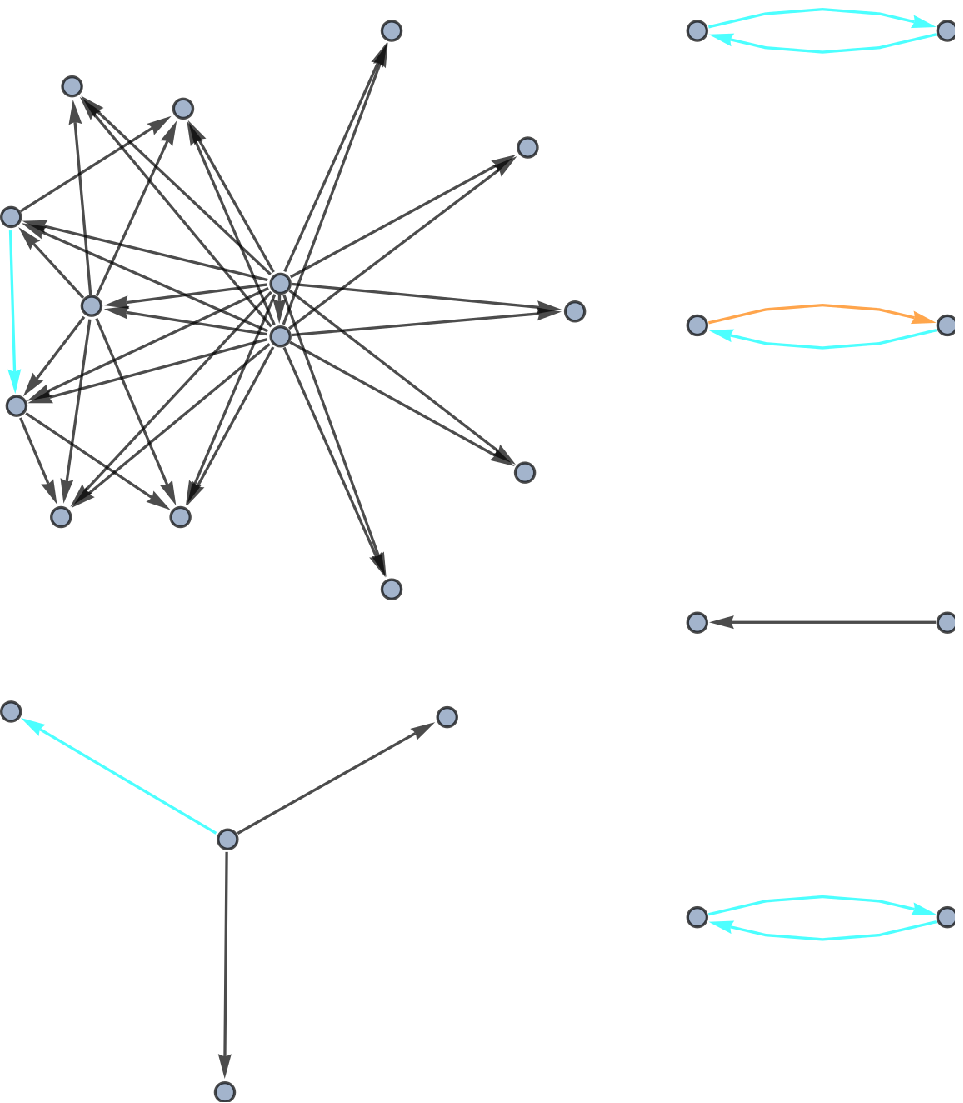}
    \caption{ctg7180000087484} 
    % joins the large component at $d=13$}
    \label{ctg7180000087484}
  \end{minipage}
 \hspace{20mm}
  \begin{minipage}[t]{.32\textwidth}
    \includegraphics[width=\textwidth]{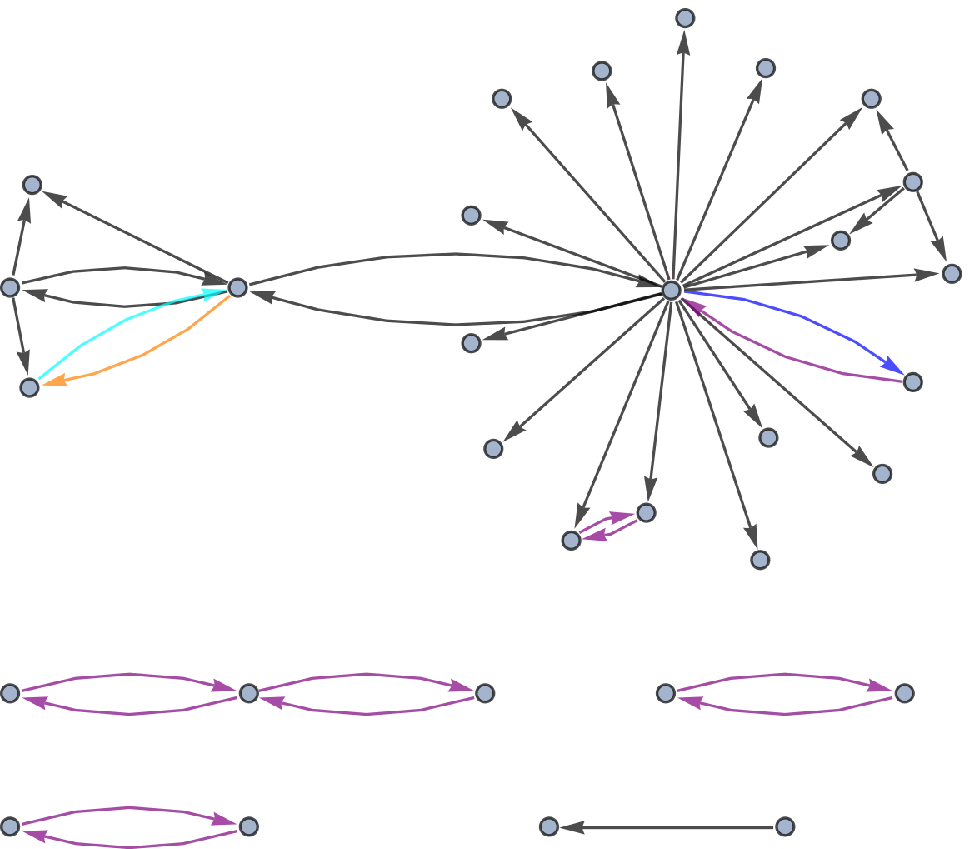}
    \caption{ctg7180000067363} % joins the large component at $d=14$}
     \label{ctg7180000067363}
  \end{minipage}
\end{figure}

%\newpage

\begin{figure}[H]  
  \begin{minipage}[t]{.32\textwidth}
    \includegraphics[width=\textwidth]{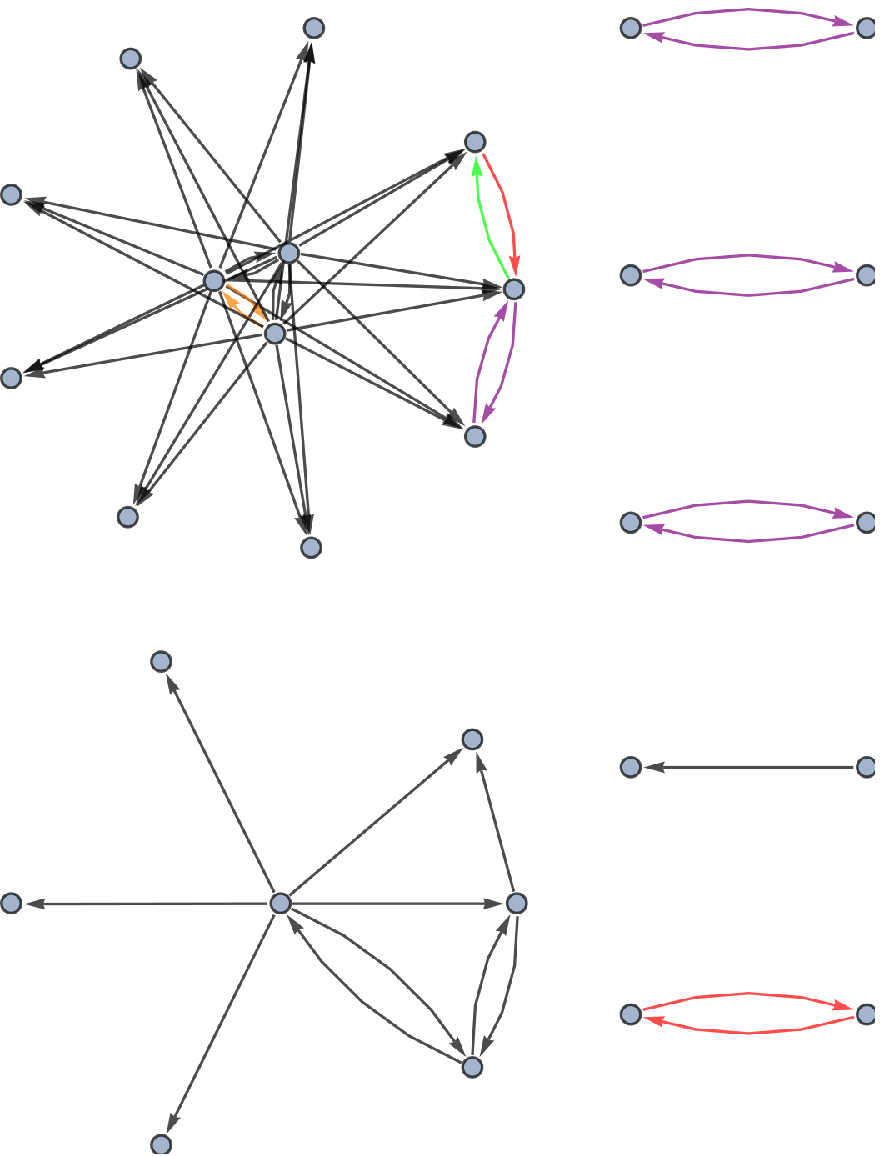}
    \caption{ctg7180000067280} % joins the component at $d=15$ }
    \label{ctg7180000067280}
  \end{minipage}
  \hspace{20mm}
   \begin{minipage}[t]{.32\textwidth}
    \includegraphics[width=\textwidth]{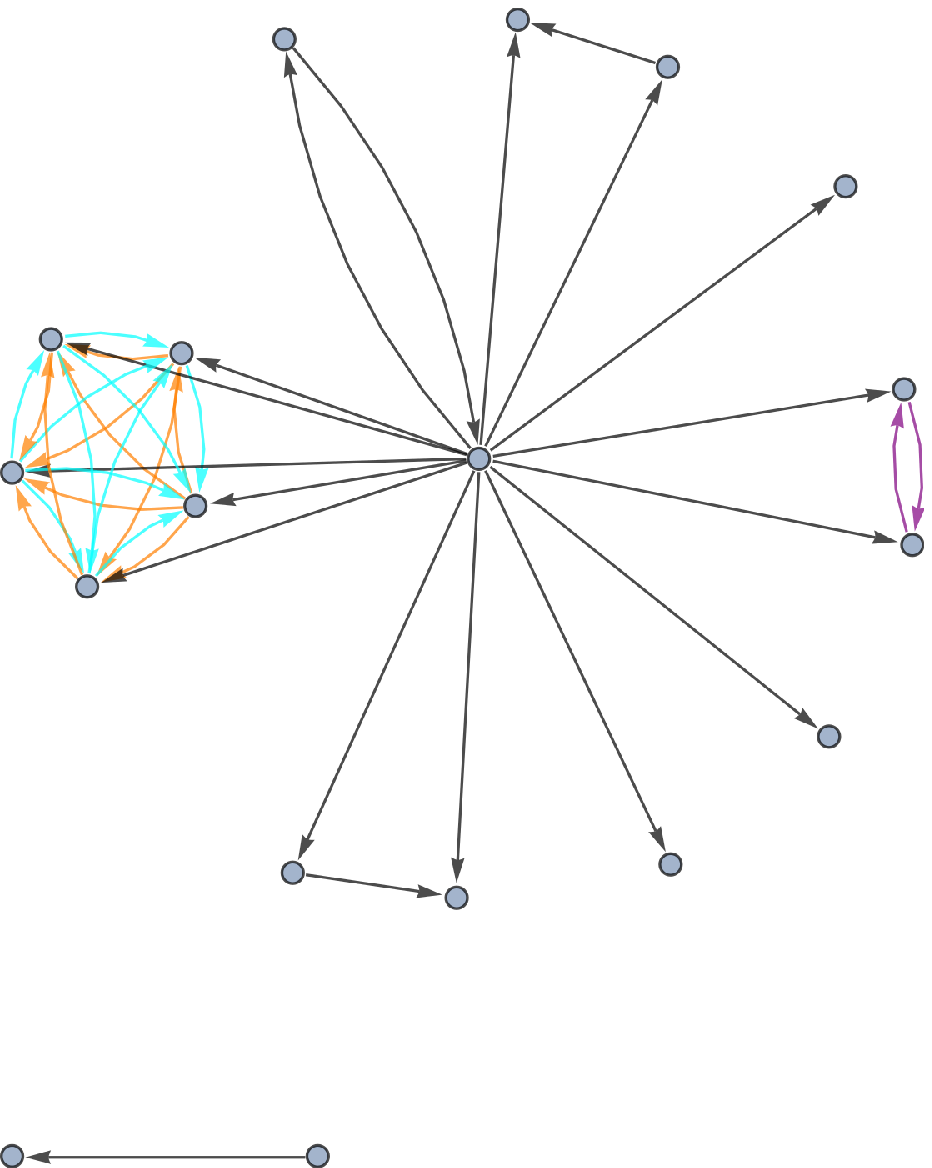}
    \caption{ctg7180000067243} % joins the large component at $d=15$}
     \label{ctg7180000067243}
   \end{minipage}    
     \end{figure}

\begin{figure}[H]
  \begin{minipage}[t]{.32\textwidth}
    \includegraphics[width=\textwidth]{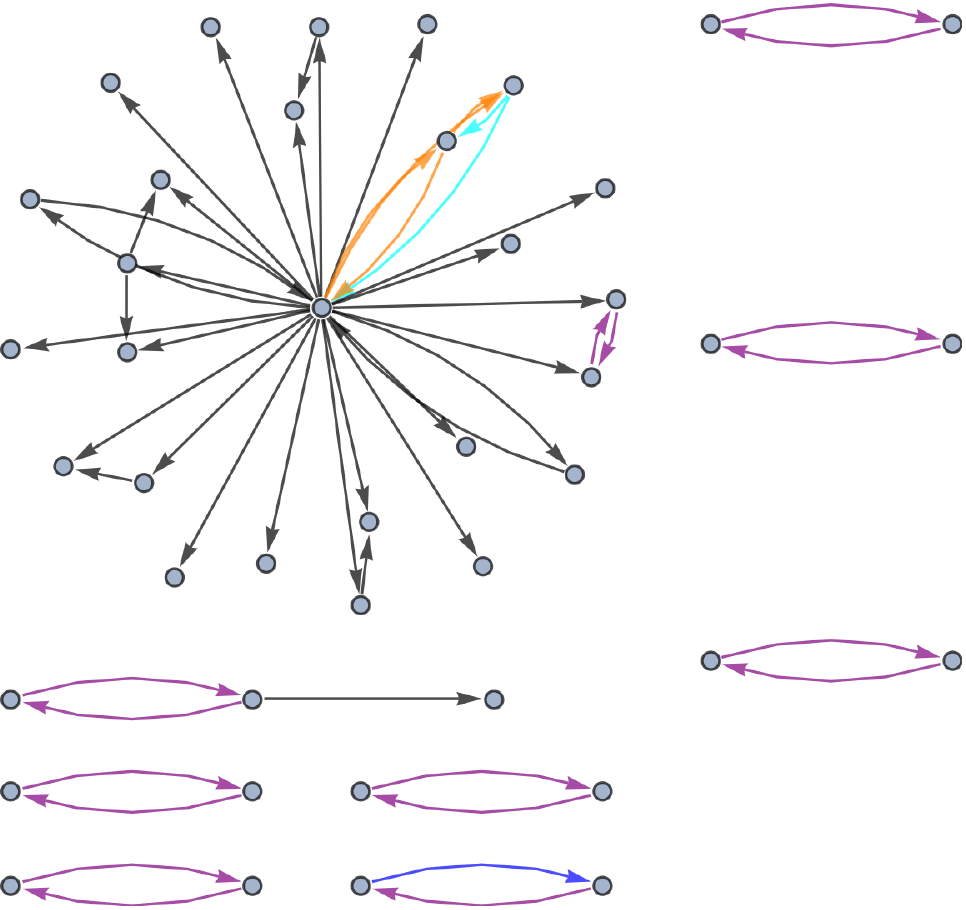}
    \caption{ctg7180000067157} % joins the large component at $d=23$}
    \label{ctg7180000067157}
  \end{minipage}
   \hspace{20mm}
   \begin{minipage}[t]{.32\textwidth}
    \includegraphics[width=\textwidth]{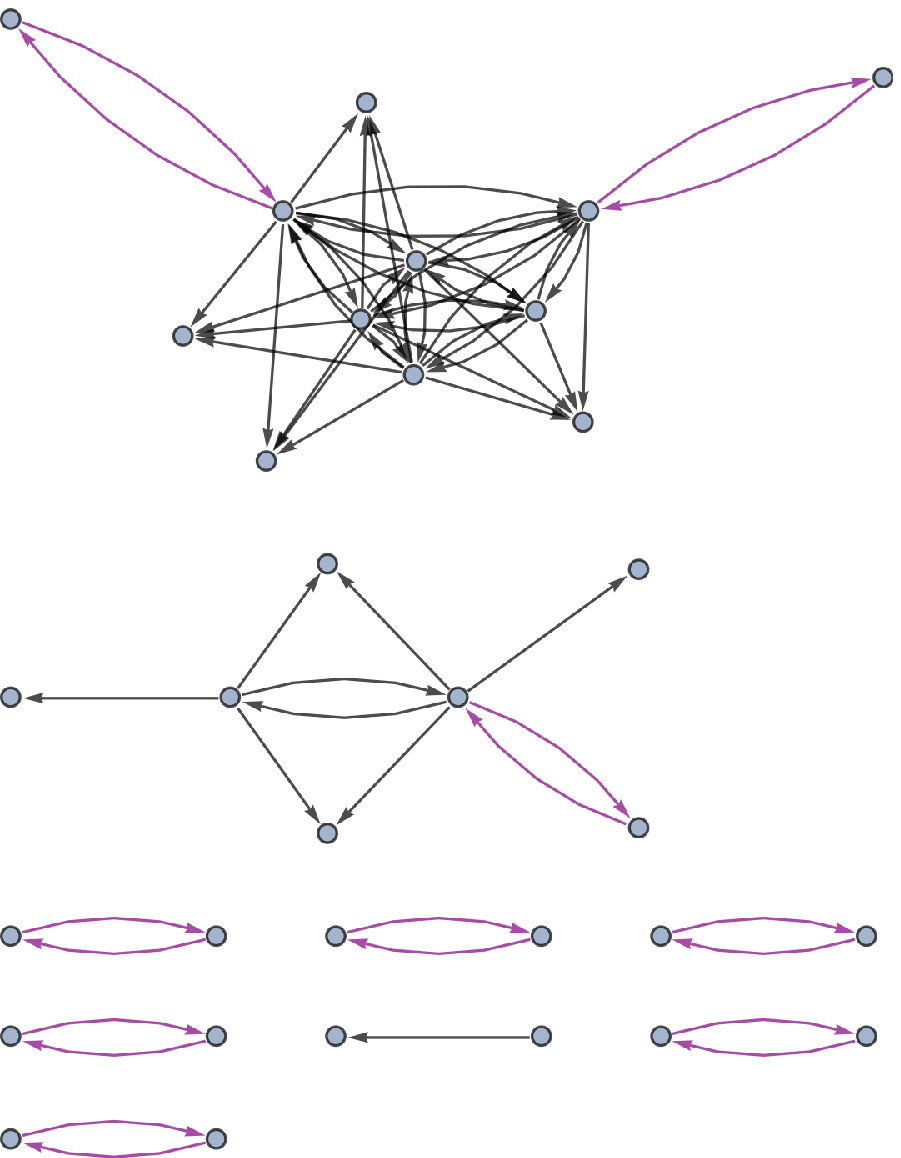}
    \caption{ctg7180000067223} % joins the large comonent last at $d=24$}
    \label{ctg7180000067223}
  \end{minipage}
\end{figure}

\begin{figure}[H]
  \begin{minipage}[t]{.32\textwidth}
    \includegraphics[width=\textwidth]{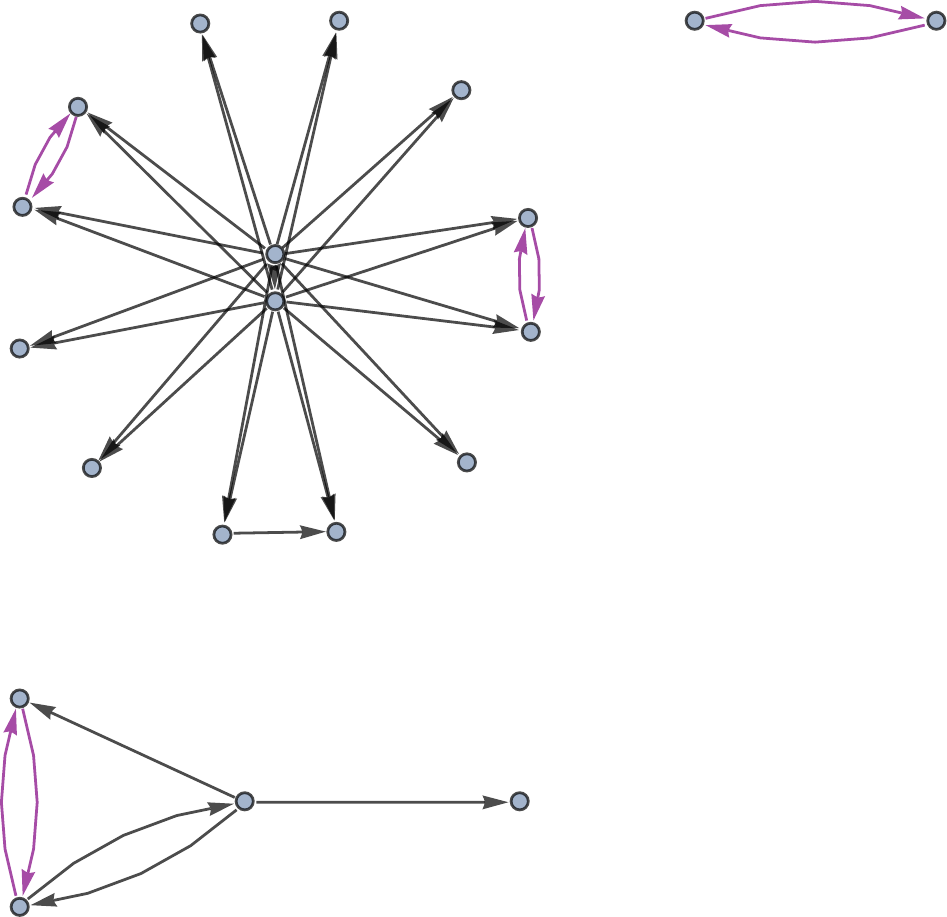}
    \caption{ctg7180000067417} 
    \label{ctg7180000067417}
  \end{minipage}
   \hspace{20mm}
   \begin{minipage}[t]{.32\textwidth}
    \includegraphics[width=\textwidth]{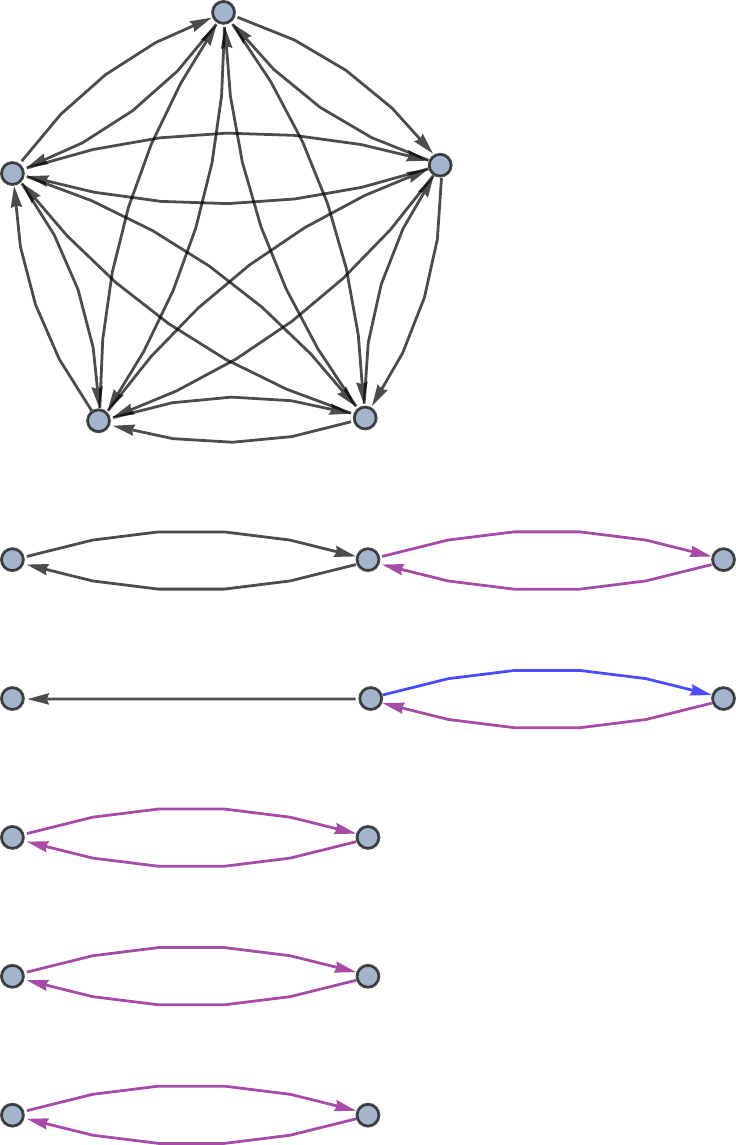}
    \caption{ctg7180000067411} 
    \label{ctg7180000067411}
  \end{minipage}
\end{figure}

\begin{comment}

\begin{figure}[H]
  \begin{minipage}[t]{.22\textwidth}
    \includegraphics[width=\textwidth]{ctg7180000067417}
    \caption{ctg7180000067417} 
    \label{ctg7180000067417}
  \end{minipage}
   \hspace{20mm}
   \begin{minipage}[t]{.22\textwidth}https://preview.overleaf.com/public/sdbcrsbbwtfr/images/67bd44b1aeccfeb46ba958cc6a7b2d7c3bc53d0a.jpeg
    \includegraphics[width=\textwidth]{ctg7180000067411}
    \caption{ctg7180000067411} 
    \label{ctg7180000067411}
  \end{minipage}
\end{figure}

\end{comment}

\end{document}